\documentclass{article}

     \usepackage[sort&compress,numbers]{natbib}
     \usepackage[final]{neurips_2021}

\usepackage[utf8]{inputenc} % allow utf-8 input
\usepackage[T1]{fontenc}    % use 8-bit T1 fonts
\usepackage{url}            % simple URL typesetting
\usepackage{booktabs}       % professional-quality tables
\usepackage{amsfonts}       % blackboard math symbols
\usepackage{nicefrac}       % compact symbols for 1/2, etc.
\usepackage{microtype}      % microtypography
\usepackage{xcolor}         % colors
\usepackage[colorlinks]{hyperref}
\hypersetup{citecolor=[RGB]{5,112,176},linkcolor=red,urlcolor=[RGB]{5,112,176}}

\usepackage{amsfonts} 
\usepackage{algorithm}
\usepackage{algpseudocode}
\usepackage{subfigure}

\usepackage{tablefootnote}

\def \endprf{\hfill {\vrule height6pt width6pt depth0pt}\medskip}

\newenvironment{proof}{\noindent {\bf Proof} }{\endprf\par}

\usepackage{mathtools}
\usepackage{bm}

\usepackage{graphicx}
\usepackage{amssymb}

\newcommand{\BE}{\mathbb{E}} % expectation
\newcommand{\R}{\mathbb{R}} % reals
\newcommand{\CS}{\mathcal{S}} % integers
\newcommand{\CA}{\mathcal{A}} %
\newcommand{\CT}{\mathcal{T}} % 
\newcommand{\CE}{\mathcal{E}} % 

\newcommand{\Tr}{{\rm Trace}} % 

\newcommand{\argmax}{\arg\max} % 

\newcommand{\Span}{\rm Span} % column space
\newcommand{\diag}{{\rm diag}} % column space
\newcommand{\dist}{{\rm dist}} % column space

\usepackage{amsthm}

\newtheorem{theorem}{Theorem}%[section]

\newtheorem{assumption}{Assumption}

\newtheorem{lemma}{Lemma}
\newtheorem{corollary}{Corollary}

\theoremstyle{definition}
\newtheorem{remark}{Remark}%[subsection]

\newcommand*{\QE}{\hfill\ensuremath{\blacksquare}}%
\newenvironment{pfof}[1]{\vspace{1ex}\noindent{\textbf{Proof of
		#1:}}\hspace{0.5em}} {\hfill\QE\vspace{1ex}}

\title{Non-Stationary Representation Learning in Sequential Linear Bandits}

\author{%
  Yuzhen Qin\thanks{Department of Mechanical Engineering,
  	University of California, Riverside, CA 92521, USA} \\
  \texttt{yuzhenq@ucr.edu}
   \And
   Tommaso Menara\thanks{Department of Mechanical and Aerospace Engineering, University of California, San Diego, La Jolla, CA 92093, USA}\\ 
   \texttt{tmenara@eng.ucsd.edu}
   \And
   Samet Oymak\thanks{Department of Electrical and Computer Engineering,
   University of California, Riverside, CA 92521, USA}\\
   \texttt{oymak@ece.ucr.edu}
   \And
   ShiNung Ching\thanks{Department of Electrical and Systems Engineering and Biomedical Engineering, Washington University in St. Louis, St. Louis, MO 63130, USA}\\
   \texttt{shinung@wustl.edu}
   \And
   Fabio Pasqualetti$^*$\\
   \texttt{fabiopas@engr.ucr.edu}
}

\begin{document}

\maketitle

\begin{abstract}
  In this paper, we study representation learning for multi-task decision-making in non-stationary environments. We consider the framework of sequential linear bandits, where the agent performs a series of tasks drawn from different environments. The embeddings of tasks in each set share a low-dimensional feature extractor called \textit{representation}, and representations are different across sets. We propose an online algorithm that facilitates efficient decision-making by learning and transferring non-stationary representations in an adaptive fashion. We prove that our algorithm significantly outperforms the existing ones that treat tasks independently. We also conduct experiments using both synthetic and real data to validate our theoretical insights and demonstrate the efficacy of our algorithm. 
\end{abstract}

\section{INTRODUCTION}
Humans are naturally endowed with the ability to learn and transfer experience to later unseen tasks. The key mechanism enabling such versatility is the abstraction of past experience into a `basis set' of simpler representations that can be used to construct new strategies much more efficiently in future complex environments \cite{FNT-FMJ,RA-SYS-NY:2021}.

Inspired by this observation, recent years have witnessed an increasing interest in the study of \textit{representation learning} \cite{BY-CA-VP:2013}. Representation learning is an important tool to perform transfer learning, wherein common low-dimensional features shared by tasks are inferred and generalized. It underlies major advances in a variety of fields including language processing \cite{LJD-LQ-SN-ZJ:2020}, drug discovery \cite{RB-KS-RP-WD-KD-PV:2015}, and reinforcement learning \cite{DC-TD-BA-RM-PJ:2019}. 
Due to its promising seminal impact, there are many recent theoretical studies on representation learning
%has sparked considerable interest
(e.g., see \cite{BMF-KM-TA:2019,DS-HW-KSM-LJD-LQ:2020,TN-JC-JMI:2021,TN_JM_JC:2020,BQ-IA-RL-AZ-ZY-HA:2020}).  Yet, existing literature focuses on representation learning for batch tasks and is restricted to static representations, thus relying on the working assumption that \emph{one representation fits all tasks}.

Most realistic decision-making scenarios feature two challenges: (i) the learning agent faces tasks that appear in sequence, and (ii) the agent may encounter distinct environments sequentially (see Fig.~\ref{conceptual}~(a)), where learning a single representation is no longer sufficient. Humans can perform extraordinarily well in such scenarios because of their flexibility to adapt to new environments. For instance, in the Wisconsin Card Sorting Task (WCST, see Fig.~\ref{conceptual}~(b)), participants are asked to match a sequence of stimulus cards to one of the four cards on the table according to some sorting rule --- number, shape, or color. The sorting rule changes every now and then without informing the participants. Earlier studies  have shown that, in general, humans perform very well on this task (e.g., \cite{BA-DW-JP-RJK:2009}). By contrast, some classical learning algorithms, such as the tabular-Q learning and the deep-Q learning, struggle in WCST (as we show in Section~\ref{sec:Experiments}). Unlike humans, these algorithms can neither abstract succinct information from experience nor adapt to new environments. This observation reveals the need to develop more human-like reasoning and a more fluid approach in representation learning. 

This paper takes an important step towards a deeper theoretical understanding of representation learning in non-stationary environments. As a prototypical sequential decision-making scenario, we consider a series of linear bandit models, where each bandit represents a different task, and the objective is to maximize the cumulative reward by interacting with these tasks. Moreover, sequential tasks are drawn from different environments. Importantly, tasks in the same environment share a low-dimensional linear representation, and different environments have their own representations. Our modeling choice can be used in a wide range of applications. For instance, consider an adaptive  system that recommends music to  users of a streaming platform. This system naturally fits our model: non-stationary environments arise due to the fact that distinct groups of active users can have different preferences at different times of the day (see Fig.~\ref{conceptual}~(c)). The goal of this paper is to analytically study representation learning in dynamical environments akin to the above recommendation system.

\begin{figure}[t]
	\centering
	\includegraphics[scale=1.8]{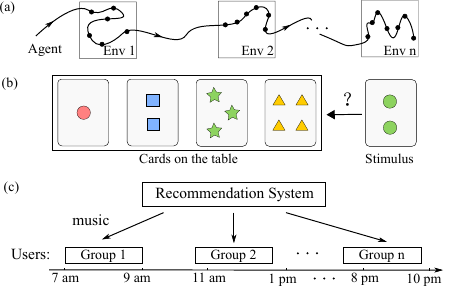}
	\caption{Scenarios with non-stationary environments. (a) Changing environments that the agent enters to perform tasks in sequence. (b) The Wisconsin Card Sorting Task, a classical test to assess cognitive capabilities that features a changing sorting rule; (c) A recommendation system that recommends music to distinct users at different times of the day  based on the users' tastes.}
	\label{conceptual}
\end{figure}

\textbf{Related Work.\;} As a well-known model to capture the exploration-exploitation dilemma in decision-making, multi-armed bandits have attracted extensive attention (see \cite{BS-CN:2012} for a survey). A variety of generalized bandit problems have been investigated, where the situations with non-stationary reward functions \cite{WL-SV:2021}, restless arms \cite{GT-CK:2020}, satisficing reward objectives \cite{RP-SV-LNE:2016}, heavy-tailed reward distributions \cite{WL-SV:2020}, risk-averse decision-makers \cite{MM-CO:2021}, and multiple players \cite{HMK-DS:2021} are considered.  Distributed algorithms have also been proposed to tackle bandit problems (e.g., see \cite{DK-NN-RJ:2014,LP-SV-LNE:2018,MRD-KV-RP:2019,LP-SV-LNE:2021,UM-NEL:2020,Zhu-Liu:2021,MU-LN:2021,MU-DA-LN-PA:2021}).Recently, some studies take into account the nature that sequentially collected data is adaptive and propose novel algorithms to further improve the performance \cite{DM-RZ-ZZ:21,ZR-RZ-AS-ZZ:21}.

Besides the above work that focuses on single-task bandits, some efforts have also been made to study multi-task problems. The core of multi-task bandits is to learn and transfer interrelationships across multiple tasks, aiming to improve decision-making efficiency compared to treating tasks independently. Various types of interrelationships can be leveraged to  boost the learning agent's performance, including the mean of tasks drawn from a stationary distribution \cite{AMG-LA-BE:2013,CL-LA-PM:2020}, similarity of task coefficients in linear bandits \cite{SM-AO-LA-PJ:2014},  and resemblance in contexts of arms in contextual bandits \cite{AAD-UD-CS:2017}. 
Recently, learning low-dimensional subspaces shared by task coefficients has also been proven to improve performance in simultaneous linear bandits \cite{Yang-Hu-Lee:2021,HJ-CX-JC-LL-WL:2021,CL-LK-PM:22}. 

\textbf{Contribution.\;} 
This paper seeks to develop methods to learn and transfer non-stationary representations across sequential bandits. In contrast to recent studies that play bandits simultaneously \cite{Yang-Hu-Lee:2021,HJ-CX-JC-LL-WL:2021}, the sequential setting is more realistic, and representation learning in this context is much more challenging. First, there does not exist a low-dimensional representation that fits all bandits. The agent needs to adapt to dynamical environments. Second, the agent does not know when environment changes happen, thus has no knowledge of the number of tasks drawn from an environment. It is therefore challenging to strike the balance between learning and transferring the representation. 

We propose an adaptive algorithm to overcome these challenges. Within each environment, this algorithm alternates between representation exploration and exploitation, making it flexible to different durations of the environments. Meanwhile, we incorporate a change-detection strategy into our algorithm to adapt to non-stationary environments. We further obtain an upper bound for our algorithm $\tilde{O} (dr \sqrt{m S N} + S r \sqrt{N})$, with $d$ being the task dimension, $r$ the representation dimension, $m$ the number of environments, $S$ the number of tasks, and $N$ the number of rounds for each task. Our regret significantly outperforms the baseline $\Theta(Sd\sqrt{N})$ of algorithms treating tasks independently. To demonstrate our theoretical results, we perform some experiments using synthetic data and LastFM data. Simulation results also show that our algorithm considerably outperforms classical reinforcement learning algorithms in WCST.

Our preliminary work \cite{YQ-TM-SO-SC-FP:22a} presents limited theoretical findings on representation learning in the sequential setting, but we go well beyond that in this paper by considering non-stationary environments and providing a comprehensive account of the results. Further, we demonstrate the broad applications of our results by presenting more experiments.

\textbf{Organization.\:} The problem setup is in Section~\ref{sec:setup}. In Section~\ref{sec:within}, we present an algorithm that performs representation learning in a single environment. An environment-change-detection algorithm is provided in Section~\ref{sec:detection}. In Section~\ref{sec:main_sec}, the main algorithm is presented by putting together Sections~\ref{sec:within} and \ref{sec:detection}. Illustrative experiments are reported in Section~\ref{sec:Experiments}. Concluding remarks appear in Section~\ref{sec:conclusion}.

\textbf{Notation.\;} Given a matrix $A\in \R^{d\times r}$, $r<d$, $\Span(A)$ denotes its column space, $A_\perp\in\R^{d\times (d-r)}$ the orthonormal basis of the complement of $\Span(A)$, $[A]_i$ its $i$th column, $\sigma_{i}(A)$ its $i$th largest singular value, and $\|A\|_F$ its Frobenius norm. We use $\|x\|$ to denote the $L_2$ norm if $x$ is a vector and the spectral norm if $x$ is a matrix.  Let $A$ and $B$ be two orthonormal basis of two subspaces $\CA,\mathcal B \subset\R^{d\times r}$. Define $\sin \bm{\theta}_1(A,B):=\sin( \bm{\theta}_1)$ and $\sin \bm{\theta}_r(A,B):=\sin( \bm{\theta}_r)$, where $\bm{\theta}_1$ and $\bm{\theta}_r$ are computed by the singular value decomposition $A^\top B=UDV^\top$ with $D=\diag(\cos \bm{\theta}_1,\dots,\cos  \bm{\theta}_r)$ satisfying $0\le  \bm{\theta}_r\le \dots \le  \bm{\theta}_1\le \frac{\pi}{2}$. Following \cite{DC-KWM:70}, the distance between $\CA$ and $\mathcal B$ is defined as $\dist(\CA,\mathcal B)=\|\diag(\sin \bm{\theta}_1,\dots,\sin  \bm{\theta}_r)\|_F=\|A^\top B_\perp\|_F$.
Given a positive number $x$, $\lceil x\rceil$ denotes the smallest integer that is greater than or equal to $x$. Given two functions $f,g:\R^+\to \R^+$, we write $f(x)=O(g(x))$ if there is $M_o>0$ and $x_0>0$ such that $f(x)\le M_og(x)$ for all $x\ge x_0$, and $f(x)=\tilde{O}(g(x))$ if $f(x)=O(g(x)\log^k (x))$. Also, we denote $f(x)=\Omega(g(x))$ if there is $M_\Omega>0$ and $x_0>0$ such that $f(x)\ge M_\Omega g(x)$ for all $x\ge  x_0$, and $f(x)=\Theta(g(x))$ if $f(x)=O(g(x))$ and $f(x)=\Omega(g(x))$.

\section{Problem Setup}\label{sec:setup}

In this paper, we consider the following  multi-task sequential linear bandits model:
\begin{align}\label{main:model}
	y_t= x_t^\top \theta_{q(t)} +\eta_t,
\end{align}
where $x_t\in \CA\subseteq \R^d$ is the action taken by the agent at round $t$, $\theta {\in \R^d}$ is the bandit coefficient, and $\eta_t$ is the additive noise that is assumed to be zero mean $1$-sub-Gaussian, i.e., $\mathbb{E} [e^{\lambda \eta_t}] \le \exp({\frac{\lambda^2}{2}})$ for any $\lambda>0$. 

Notice that the coefficient vector $\theta_{q(t)}$ is time-varying. 
We assume that $q(t)=\lceil \frac{t}{N}\rceil$, where $t=1,2,\dots,SN$. That is, the agent  plays $S$ bandits in sequence and interacts with each bandit for $N$ rounds\footnote{We make this assumption for simplicity. Our results can be readily generalized to the case where bandits are played for different rounds.}. Then, the task sequence can be denoted as  $\CS:=\{\theta_1,\theta_2,\dots,\theta_{S}\}$.	
	
	We further assume that these tasks are drawn from $m$ different environments, $\CE_1,\CE_2,\dots,\CE_m$ (each $\CE_{k}$ is a set of tasks). Specifically, the task sequence $\CS$ can be divided into $m$ consecutive subsequences (see Fig.~\ref{Pro:Setup}), i.e.,
	\begin{align*}
		\CS=\{\CS_1,\CS_2,\dots,\CS_m\},
	\end{align*}
	such that $\CS_k\subseteq \CE_k$ for $k=1,2,\dots,m$. Denote $\tau_k$ as the number of tasks in $\CS_k$; it satisfies $\sum_{i=k}^{m}\tau_k=S$. We assume that, for each environment $\CE_k$,  there exists a matrix $B_k\in \R^{d\times r_k}$ with orthonormal columns such that
\begin{align*}
	\exists \alpha_i\in \R^{r_k}:\theta_i=B_k \alpha_i
\end{align*}
for  any $\theta_i \in \CE_k$. This assumption is motivated by the fact that real-world tasks often share low-dimensional structures, called \textit{representations} \cite{BY-CA-VP:2013}. As for the examples in Fig.~\ref{conceptual}, the representation can describe the common preferences of a user group on a music streaming platform or a certain sorting rule of the Wisconsin sorting task. With a slight abuse of terminology, we refer to $B_k$ as the representation of the $k$th  environment, $k=1,\dots,m$. For simplicity, we assume that the representations have the same dimension\footnote{Our results can be applied to the situation with heterogeneous $r_k$ by simply letting $r=\max_k r_k$.}, i.e., $r_k=r$ for all $k=1,\dots,m$.

The goal is to maximize the cumulative reward $\sum_{t=1}^{SN} y_t$ by interacting with the sequential bandits in the non-stationary environments. The agent knows $\sigma(t),N,d$, and $r$, but has no knowledge of $\theta_i$, $B_k$, and $\tau_k$ for any $i=1,\dots,S$ and $k=1,\dots,m$.
To measure the agent's performance in the $T:=SN$ rounds, we introduce the (pseudo-)\textit{regret}
\begin{align}\label{regret}
	R_{T}= \sum_{t=1}^{T} (x_{t}^*-x_{t})^\top\theta_{{q(t)}},
\end{align}
where $x^*_{t} =\argmax_{x \in \mathcal A} x^\top \theta_{{q(t)}}$ is the optimal action that maximizes the reward at round $t$. Maximizing the cumulative reward is then equivalent to minimizing the regret $R_T$. 

Following existing studies (e.g., see \cite{RP-TJN:2010,LY-WY-CX-ZY:2021}), we assume that the following  assumptions on the action set $\mathcal A$ and the task coefficients $\theta_i$  hold throughout this paper. 
\begin{assumption}[Linear bandits]\label{Assump_1}	We assume that: (a) the action set $\CA$ is the ellipsoid of the form $\{x\in \R^d:x^\top M^{-1}x\le 1\}$, where $M$ is a symmetric positive definite matrix, and (b) there are positive constants $\theta_{\min}$ and $\theta_{\max}$ so that $\theta_{\min}\le \|\theta_s\| \le \theta_{\max}$ for all  $s \in \{1,2,\dots,S\}$.
\end{assumption}

\begin{figure}[t]
	\centering
	\includegraphics[scale=0.7]{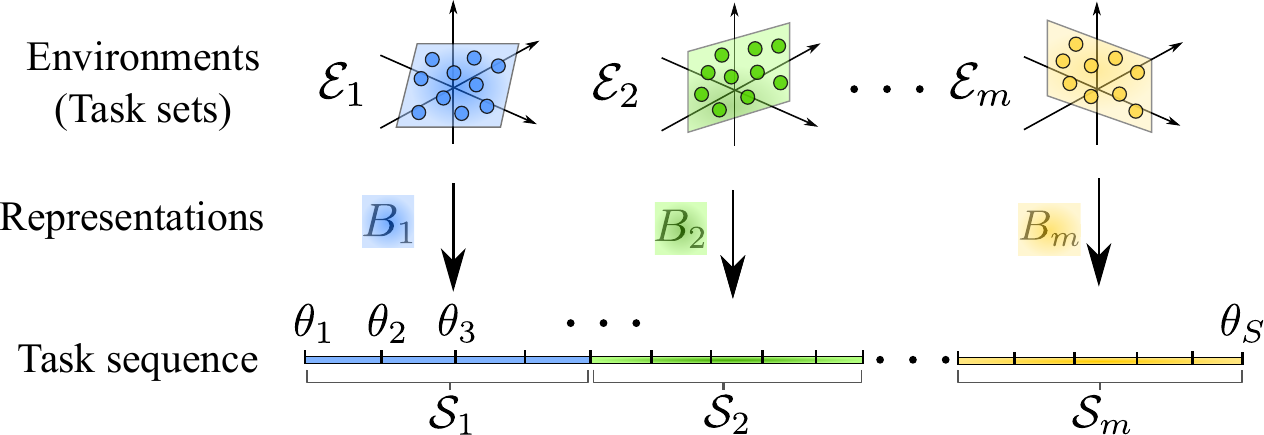}
	\caption{Sequential tasks in non-stationary environments. The tasks are taken from different environments $\CE_1,\dots,\CE_m$, forming subsequences $\CS_1,\dots,\CS_m$. The tasks in each environment share a low-dimensional representation (descried by $B_k,k=1,\dots,m$). In each subsequence $\CS_k$, there are $\tau_k$ tasks, but $\tau_k$'s are not known \textit{a priori}.}
	\label{Pro:Setup}
\end{figure}

\textbf{Limitation of Independent Strategies.} Under the stated assumptions, previous work (e.g., see \cite{Dani-Hayes-Kakade:2008,RP-TJN:2010,LY-WY-CX-ZY:2021}) shows that the regret of single-task bandits is lower bounded by $\Omega(d\sqrt{d})$. Intuitively, if the sequential bandits $\CS=\{\theta_1,\theta_2,\dots,\theta_S\}$ are independently played by those standards algorithms, the best performance is optimally $\Theta(Sd\sqrt{N})$.
% (i.e., when the upper bound matches the lower bound). 

\textbf{Potential Benefits of Representation Learning.} In the case of an oracle where the representations $B_1,\dots,B_m$ are known, it holds that $\theta_i=B_k\alpha_i$ for any $\theta_i\in\CS$. Letting $z_t=B_k^\top x_t$, the $d$-dimensional bandit $y_t=x_t^\top \theta_i+\eta_t$ becomes a $r$-dimensional one $y_t=z_t^\top \alpha_i+\eta_t$ with $\alpha_i\in\R^r$. Following \cite{Dani-Hayes-Kakade:2008,RP-TJN:2010,LY-WY-CX-ZY:2021}, one can show that the best performance of an algorithm can be optimally $\Theta(Sr\sqrt{N})$, which indicates a significant performance improvement compared to the standard algorithms if $r\ll d$. The reason is that learning to make decisions is accomplished in much lower-dimensional subspaces.  
The above observation implies potential benefits of representation learning, that is, exploring and exploiting the underlying low-dimensional structure between bandit tasks can facilitate more efficient decision-making. 

In our setting, representation learning has two main challenges. First, how can the agent explore and exploit representation in the sequential setting? Particularly, striking the balance between exploration and exploitation becomes more challenging than the situation where bandits are played concurrently \cite{Yang-Hu-Lee:2021,HJ-CX-JC-LL-WL:2021}. There is a trade-off between the need to explore more sequential tasks (more data samples) to obtain a more accurate representation estimate and the incentive to exploit the learned representation for more efficient learning and higher immediate rewards.   Second, how can the agent deal with  environment changes? The remainder of this paper aims to address these challenges.

\section{Representation Learning in Sequential Bandits: Within-environment  Policy} \label{sec:within}
In this section, we show how the agent can improve its performance by using representation learning in the sequential setting. The main result is the sequential representation learning algorithm (SeqRepL, see Algorithm~\ref{alg:SeqRepL}), a within-environment policy that deals with individual segments of tasks drawn from the same environment. 
\subsection{Sequential Representation Learning Algorithm}

The key feature of SeqRepL is to balance representation exploration and exploitation without knowing the duration of each environment. 
To show how SeqRepL works, we first restrict our attention to a series of bandit tasks in this section:
\begin{align}\label{Sequ_tasks}
	\CT= \{ \theta_1, \dots, \theta_\tau\},
\end{align}
where the number of tasks $\tau$ is unknown, and  the representation shared by the tasks is $B \in \R^{d \times r}$.

\begin{algorithm}[tb]
	\caption{Representation Exploration (RepE)}
	\label{alg:RE}
	\begin{algorithmic}[1]
		\footnotesize
		\State \textbf{Input:}  $N$, $N_1$, \{$a_1,\dots,a_d$\}.
		\For{$t=1: N_1$}
		\State {Take the action $x_t= a_i$, $i=({t-1\mod d})+1$}
		\EndFor
		\State Compute $\hat \theta=(X_{\rm E} X_{\rm E}^\top)^{-1}X_{\rm E} Y_{\rm E}$
		\For {$t=N_1+1: N$} 
		\State Take the action $x_t= \arg\max_{x \in \mathcal A} x^\top \hat \theta$ 
		\EndFor
	\end{algorithmic}
\end{algorithm}
\begin{algorithm}[tb]\caption{Representation Transfer (RepT($\hat B$))}\label{alg:RT}
	\begin{algorithmic}[1]
		\footnotesize
		\State \textbf{Input:} $N$, $N_2$, $\hat B \in \R^{d\times r}$,  $\{a_1',\dots,a_r'\}$
		\For {$t=1: N_2$}
		\State Take the action $x_i= a'_i$, $i=({t-1\mod  r})+1$
		\EndFor
		\State Compute 
		$\hat \alpha=(\hat B^\top X_{\rm T}X_{\rm T}^\top \hat B)^{-1}\hat B^\top X_{\rm T}Y_{\rm T}$ and $\hat \theta= \hat B \hat \alpha$
		\For {$t=N_2+1: N$}
		\State Take the action $x_t= \arg\max_{x \in \mathcal A} x^\top \hat \theta$
		\EndFor 
	\end{algorithmic}
\end{algorithm}

\begin{algorithm}[t]\caption{SeqRepL}\label{alg:SeqRepL}
	\begin{algorithmic}[1]
		\footnotesize 
		\State \textbf{Input:} $L$, $N_1$, $N_2$, $n$. \hspace{1.5cm}\textbf{Initialize}: $\hat P=0_{d\times d}$.
		\State \textbf{for} cycle $n=1,2,\dots$ \textbf{do}  
		\State \hspace{5pt} \textit{Rep Exploration:} play $L$ tasks in $\CT$ using RepE,\\ \hspace{2.5cm} $\hat P=\hat P+\hat\theta_i \hat \theta_i^\top$ \hfill{\scriptsize $\triangleright$ fixed duration: $L$}
		\State \hspace{5pt} $\hat B \leftarrow$ the left-singular vectors associated with the largest $r$ singular values of $\hat W=\frac{1}{nL}\hat P$ %top $r$	singular value decomposition of $\hat W=\frac{1}{nL}\hat P$
		\State 	\hspace{5pt} \textit{Rep Transfer:} play $nL$ tasks in $\mathcal S_{\tau}$ using RepT($\hat B$)  \hfill {\scriptsize $\triangleright$ increasing duration: $nL$}\\
		\textbf{end for}
	\end{algorithmic}
\end{algorithm}

SeqRepL operates in a cyclic manner, which is inspired by the PEGE algorithm \cite{RP-TJN:2010}. It alternates between two sub-algorithms---\textit{representation exploration} (RepE, see Algorithm~\ref{alg:RE}) and \textit{representation transfer} (RepT, see Algorithm~\ref{alg:RT}). Both RepE and RepT are explore-then-commit (ETC) algorithms, consisting of two stages, i.e., {exploration} and {commitment}.

\textbf{RepE.} At the exploration stage of $N_1$ rounds, $d$ actions, $a_1,\dots,a_d$, are repeatedly taken in sequence. These actions can be arbitrarily chosen but need to be linearly independent such that they span the action space defined by $\CA$. In this paper, we simply let $a_i=\lambda_0 e_i$, where $e_i$ is the $i$th  canonical vector of $\R^d$ and $\lambda_0>0$ is such that $a_i \in \CA$ for all $i$. Then, the coefficient $\theta$ is estimated by the least-squares regression
\begin{align*}
	\hat \theta=(X_{\rm E} X_{\rm E}^\top)^{-1}X_{\rm E}Y_{\rm E},
\end{align*}
where $X_{\rm E}=[x_1,\dots,x_{N_1}]$ and $Y_{\rm E}=[y_1,\dots,y_{N_1}]^\top$ respectively collect the actions and rewards at this stage.  At the commitment stage, the greedy action $x_t=\argmax_{x\in \CA} x^\top\hat\theta$ is taken for $N-N_1$ times.  

\textbf{RepT.} Different from RepE, RepT utilizes $\hat B \in \R^{d \times r}$ as a plug-in surrogate for the unknown representation $B$ to learn the coefficient $\theta$.  
The exploration stage is accomplished in the $r$-dimensional space $\Span(\hat B)$. Specifically, $r$ actions $a'_1,\dots,a'_r$, are repeatedly taken for $N_2$ rounds. These actions can be arbitrarily chosen in $\Span(\hat B)$ such that they are linearly independent. In this paper, we let $a_i'=\lambda_0 [\hat B]_i$, where $\lambda_0>0$  is such that $a_i'\in \CA$. To estimate $\theta$, RepT computes $\alpha$ first. Observe that $y_t=x_t^\top \theta+\eta_t =x_t^\top B\alpha+\eta_t$. Then, $\alpha$ is estimated with $\hat B$ by the least-squares regression
\begin{align*}
	\hat \alpha=(\hat B^\top X_{\rm T}X_{\rm T}^\top \hat B)^{-1}\hat B^\top X_{\rm T}Y_{\rm T},
\end{align*} 
where $X_{\rm T}=[x_1,x_2,\dots,x_{N_2}]$ and $Y_{\rm T}=[y_1,y_2,\dots,y_{N_2}]^\top$. Subsequently, $\theta$ is estimated by $\hat \theta=\hat B \hat \alpha$. Similar to RepE, $x_t=\argmax_{x\in \CA} x^\top\hat\theta$ is taken at the commitment stage.

\textbf{SeqRepL.} As shown in Algorithm~\ref{alg:SeqRepL}, there are two phases in each cycle of SeqRepL. Specifically, for each cycle $n\in\{1,2,\dots\}$, these phases are:

1) \textit{Representation Exploration} phase: $L$ tasks ($L$ is to be designed) are played using RepE. Let $\hat W=\frac{1}{nL}\sum_{i} \hat \theta_i \hat \theta_i^\top$ where $\hat \theta_i$'s are all the estimated coefficients obtained by RepE in all the previous $n$ cycles. Then, the representation $B$ is estimated by performing the singular value decomposition (SVD) to $\hat W$. Specifically, $\hat B$ takes the singular vectors associated with the $r$ largest singular values of $\hat W$, i.e., $\hat B=\hat U_1$ with $\hat U_1 \in \R^{d\times r}$ taken from the SVD:  $\hat W=[\hat U_1,\hat U_2]\hat\Sigma \hat V^\top$. 

2) \textit{Representation Transfer} phase: the latest representation estimate $\hat B$ is transferred. Specifically, $nL$ sequential tasks are played using RepT($\hat B$).

Notice that for any cycle $n$, $L$ more tasks are played using RepT than the previous $(n-1)$th cycle. We next show that this alternating scheme balances representation exploration and exploitation excellently.

\begin{assumption}[Task diversity]\label{Assump:diversity}
	Suppose that there exist an integer $\ell$ and a constant $\nu>0$  such that any subsequence of length $\ell$ in the sequence in Eq.~\eqref{Sequ_tasks} satisfies $\sigma_{r}(W_{s}W_{s}^\top /\ell ) \ge \frac{\nu}{r}>0$ for any $s$, where $W_{s}=[\theta_{s+1}, \dots,\theta_{s+\ell}]$.
\end{assumption}

This assumption states that the sequential tasks well spread the entire $r$-dimensional subspace. It ensures that this subspace can be reconstructed before all the bandit tasks are played, which is crucial to allow for transfer learning in the sequential setting. A similar assumption is found in \cite{Yang-Hu-Lee:2021}, wherein bandits are played concurrently. Representation learning in the sequential setting is more challenging, thus our assumption is slightly stronger.    

\begin{theorem}[Upper bound of SeqRepL]\label{regret:seq:single_B}
	Let the agent play the series of bandits in Eq.~\eqref{Sequ_tasks} using SeqRepL in Algorithm~\ref{alg:SeqRepL}.  Select an $L$ such that\footnote{Here, the exact knowledge of $\ell$ is not required; instead, knowing the order of $\ell$ is sufficient. In practice, the assumption can be further relaxed. In Fig.~\ref{Diff_L}, we will show that a wide range of $L$ can be chosen without knowing $\ell$, while still guaranteeing the performance of our algorithms.} $L =\Theta(\ell)$, and let $N_1=dr\sqrt{{N}/{L}}$ and $N_2=r\sqrt{N}$. Then, the regret $R_{\tau N}$ of SepRepL satisfies 
	\begin{align}\label{upper_bound}
		\BE [R_{\tau N}]= {\tilde O \Big({d  r \sqrt{\tau N}}+ {\frac{d  r}{\lambda_0^2 \nu^2} \sqrt{\tau N}}+  \frac{d \ell}{r} \sqrt{\tau N } + \tau  r \sqrt{N} \Big)}.
	\end{align}
\end{theorem}

The third term in \eqref{upper_bound} is the regret incurred when transferring the oracle representation, and the first three terms include the regret that results from representation exploration and transferring the estimated representation with errors. 

\begin{remark}[Performance comparison]\label{remark:Th1}
	Recall that if the same series of bandits are played using standard algorithms that play bandits independently, e.g., UCB \cite{Dani-Hayes-Kakade:2008}, PEGE \cite{RP-TJN:2010}, and ETC \cite{AYY-AA-SC:2009}, the best regret is optimally $\Theta(\tau d\sqrt{N})$. Compared to these algorithms, SeqRepL can have better or worse performance (i.e., ``positive'' or ``negative'' transfer), depending on the properties of the bandit tasks in \eqref{Sequ_tasks}:
		\begin{itemize}
			\item If $\tau \gg \max\{r^2, \frac{r^2}{\lambda_0^4\nu^4},\frac{\ell^2}{r^2}\}$ and $d\gg r$, using SeqRepL can significantly improve the performance\footnote{We note that positive transfer can still occur in practice without satisfying this inequality. In Figs.~\ref{synthetic} and \ref{Diff_L}, we let $\tau=400$ for $\lambda_0=1,\ell=r=3$, and $\nu\approx 0.01$, much smaller than $9\times 10^{8}$ required by  the inequality here. Nevertheless, our algorithms still outperform the standard ones significantly.};
			\item If $\tau \ll \max\{r^2, \frac{r^2}{\lambda_0^4\nu^4},\frac{\ell^2}{r^2}\}$, the bound in  \eqref{upper_bound} implies that the cost of learning the representation can overwhelm the possible benefits of transfer learning. Then, using SeqRepL may result in a situation of negative transfer.
	\end{itemize}	
	Our algorithm is particularly advantageous over the standard ones when there are a large number of tasks in the sequence. Also, in sharp contrast to existing bandit algorithms using representation learning, e.g., (\cite{Yang-Hu-Lee:2021,HJ-CX-JC-LL-WL:2021}) our algorithm requires no knowledge of $\tau$.
\end{remark} 

The following corollary provide an upper bound of SeqRepL if $\ell$ in Assumption~\ref{Assump:diversity} is of the order of $r^2$. 

%Notice that the upper bound in \eqref{upper_bound} reduces to $\BE R_{\tau N}= \tilde O \big({d  r \sqrt{\tau N}} + \tau  r \sqrt{N} \big)$ if $\ell <r^2$, and to $\BE R_{\tau N}= \tilde O \Big(\frac{d \ell}{r} \sqrt{N \tau} + \tau  r \sqrt{N} \Big)$ if $\ell > r^2$. 

\begin{corollary}\label{Coro_seq}
	Assume that $\ell$ in Assumption~\ref{Assump:diversity} is of the order of $r$, i.e., $\ell=\Theta(r^2)$. Let $N_1=d\sqrt{r N}$ and $N_2=r\sqrt{N}$. Then, the regret of SepRepL for the sequential bandits in Eq.~\eqref{Sequ_tasks} satisfies $\BE [R_{\tau N}]= \tilde O\left( d r \sqrt{\tau N}+ {\frac{d  r}{\lambda_0^2 \nu^2} \sqrt{\tau N}}+  \tau r \sqrt{N} \right)$. 
\end{corollary}

If $\frac{1}{\lambda_0^2 \nu^2}=\Theta(1)$, the regret bound becomes $\BE [R_{\tau N}]= \tilde O\left( d r \sqrt{\tau N}+  \tau r \sqrt{N} \right)$. For the series of bandits in Eq.~\eqref{Sequ_tasks}, it follows from \cite{Yang-Hu-Lee:2021} that the lower bound is $\Omega (d\sqrt{r\tau  N}+\tau r \sqrt{N})$. Note that there is just a gap of $\tilde{O}(\sqrt{r})$ between our upper bound and this lower bound.

\subsection{Analysis of Theorem~\ref{regret:seq:single_B}}

We first provide some instrumental results, and we refer the readers to Appendix~\ref{pf:lm1}-\ref{pf:Th2} for their proofs.

\begin{lemma}[Regret of RepE]\label{regret:RepL}
	Given a bandit task $\theta\in \mathcal T$,  let the agent play it using RepE in Algorithm~\ref{alg:RE} for $N$ rounds. Then, the regret $R_N$ of RepE satisfies $\BE [R_N] =  O(N_1+\frac{N}{N_1}d^2)$. 
\end{lemma}

\begin{lemma}[Regret of RepT]\label{transfer:single}
	Given a bandit task $\theta\in \mathcal T$, assume that there exists $B\in \R^{d\times r}$ with orthonormal columns such that $\theta = B \alpha$ for some $\alpha\in \R^{r}$. Assume that an estimate $\hat B$ is known and satisfies $\|\hat B^\top B_\perp\|_F  \le \varepsilon$. Let the agent play this task for $N$ rounds using RepT($\hat B$) in Algorithm~\ref{alg:RT}, then the regret satisfies $\BE [R_N]= O( N_2+ \frac{N }{N_2}r^2+ N \varepsilon^2 ). $
\end{lemma}

%For the same task $\theta$, earlier studies, e.g., \cite{RP-TJN:2010,LY-WY-CX-ZY:2021,Dani-Hayes-Kakade:2008}, have shown the optimal regret is $\Theta(d\sqrt{N})$ if there is no information of the representation. By contrast, Lemma~\ref{transfer:single} states that with the knowledge of an estimated low-dimensional representation $\hat B$, the regret can be significantly reduced provided $\hat B$ is sufficiently accurate (i.e., small $\varepsilon$), indicating the advantages of representation learning. 

%The next theorem provides the accurate of $\hat B$ in each cycle of SeqRepL. 

\begin{theorem}[Accuracy of learned representation]\label{dist:B}
	Let the agent play the series of bandits in Eq.~\eqref{Sequ_tasks} using SeqRepL in Algorithm~\ref{alg:SeqRepL}. Then, for any cycle $n$, the estimate $\hat B$ at the end of the representation exploration phase satisfies
	\begin{align}
		\|\hat B^\top B_\perp\|_F \le \tilde O\left( \frac{d r}{\lambda_0 \nu}{\sqrt{\frac{1}{ nL N_1}}}  \right), \label{B:accuracy}
	\end{align}
	with probability at least $1-\frac{1}{kN_1}$.
\end{theorem}

Recall that $\|\hat B^\top B_\perp\|_F$ measures the distance between $\hat B$ and the true representation $B$. This distance decreases with $n$, implying that $\hat B$ becomes progressively more accurate as more tasks are explored by the RepE algorithm. 

%We are now ready to prove Theorem~\ref{regret:seq:single_B}. 

\begin{pfof}{Theorem~\ref{regret:seq:single_B}}	
	In the $n$th cycle of SepRepL, $L$ tasks are played in the representation exploration phase. Then, it follows from Lemma~\ref{regret:RepL} that the regret in this phase, denoted as $R_{\rm RepE} (n)$, satisfies
	$
	\BE R_{\rm RepE}(n) = O \big(L N_1+ L \frac{N}{N_1}d^2\big).
	$
	From Theorem~\ref{dist:B}, we have 
	\begin{align*}
		\|\hat B^\top B_\perp\|_F \le \tilde O \Big( \frac{d r}{\lambda_0 \nu}{\sqrt{\frac{1}{ n L  N_1}}}  \Big).
	\end{align*}
	Then, $n L$ bandit tasks are played utilizing the RepT($\hat B$) algorithm.  It follows from Lemma~\ref{transfer:single} that the regret in the RepT phase, denoted as $R_{\rm RepT} (n)$, satisfies
	\begin{align*}
		\BE [R_{\rm RepT} (n)] &=  \tilde O \big( {n L } N_2 + nL\frac{N}{N_2}r^2 + {n L} N \frac{d^2 r^2 }{\lambda_0^2 \nu^2}{\frac{1}{ n L  N_1}} \big)\\
		&=  \tilde O \big({n L } N_2 + nL \frac{N}{N_2}r^2 +{\frac{d^2 r^2}{\lambda_0^2 \nu^2}} \frac{N}{N_1} \big).
	\end{align*}
	Observe that there are at most $\bar L=\lceil \sqrt {{2 \tau}/{L}}\rceil$ cycles in the series \eqref{Sequ_tasks} since $L \bar L+L\bar L(\bar L+1)/2 \ge \tau$. 	Summing up the regret in the representation exploration and exploitation phases in all the cycles, we obtain
	\begin{align*}
		\BE [R_{\tau N}] &= {\sum_{n=1}^{\bar L} \big( R_{\rm RepL} (n) + R_{\rm RepT} (n) \big) }\\
		&\le {\tilde O \Big(  \bar L  \big(L N_1+ L \frac{N}{N_1}d^2\big)+\sum_{n=1}^{\bar L} \big( {n L } N_2 + nL \frac{N}{N_2}r^2 + {\frac{d^2 r^2}{\lambda_0^2 \nu^2}} \frac{N}{N_1}\big) \Big)}.
	\end{align*}
	Since $N_1 = {d r \sqrt{N/L}}$, $N_2=r\sqrt{N}$ and $\bar L=\lceil \sqrt {{2 \tau}/{L}}\rceil$, we have
	$
	\BE [R_{\tau N}] = \tilde O \big({d  r \sqrt{\tau N}}+ {\frac{d  r}{\lambda_0^2 \nu^2} \sqrt{\tau N}}+  \frac{d L}{r} \sqrt{N \tau} + \tau  r \sqrt{N} \big).
	$
	Then, \eqref{upper_bound} follows from $L=\Theta(r^2)$. 
\end{pfof}

\section{Environment  Change Detection}\label{sec:detection}

To handle environment  changes, we propose the representation change detection algorithm (RepCD, see Algorithm~\ref{alg:RepCD}). It is the key to endow the agent with adaptability. 

\subsection{Representation change detection algorithm}
To show how RepCD works, we consider the environment  change where the representation switches from $B\in \R^{d \times r}$ to $\bar B \in \R^{d \times r}$ (see Fig.~\ref{Detection}~(a)). 
To detect this environment  change, we seek for the tasks that do not belong to the subspace $\Span(B)$. To infer whether a task is an outlier to $\Span(B)$, RepCD takes some \textit{probing actions} and monitors the rewards. The probing actions need to ensure: (1) when a task is not in $\Span(B)$, it can be detected as an outlier with high probability; and (2) when a task is in $\Span(B)$, it can be falsely detected as an outlier with low probability. 

The key idea is to select probing actions in the \textit{orthogonal complement} $\Span(B_\perp)$, which is illustrated in Fig.~\ref{Detection}~(b).
A task $\theta$ is in the subspace  $\Span(B)$ \textit{if and only if} $B^\top_\perp \theta=0$. To generate an accurate test, all the directions defined by the columns of $B_\perp$ need to be covered by the probing actions. A naive strategy is to choose $d-r$ actions by simply exhausting the columns of $B_\perp$, i.e., let $x_i=\lambda_0 [B_{\perp}]_i^\top, i = 1,\dots, d-r$, where $\lambda_0>0$ is such that $\lambda_0 [B_{\perp}]_i\in \CA$. Taking these actions, the agent is expected to receive rewards satisfying $y_i=\lambda_0[B_{\perp}]_i^\top\theta+\eta_i=\eta_i$ if $\theta\in \Span(B)$. If the agent receives some rewards that exceed the noise level, the task $\theta$ is likely an outlier of the current representation. 

However, one may not need as many as $d-r$ probing actions if the environment  change happens between two very different representations. It is also possible that more than $d-r$ probing actions are required to detect a more subtle environment  change. Next, we show how to choose probing actions by taking both of these situations into account.

First, let $n_{\rm det}$ be the number of probing actions (we will show how to select $n_{\rm det}$ soon). Observe that any $n_{\rm det}$ can be rewritten into $n_{\rm det}=k(d-r)+\bar n$, where $k$ can be $0,1,2\dots$. The first $k(d-r)$ probing actions simply take the actions $\{x_i=\lambda_0 [B_{\perp}]_i,i=1,\dots,d-r\}$ for $k$ times. How to choose the remainder of $\bar n$ actions is more interesting. 

We require all the $d-r$ directions in $\Span(B_\perp)$ to be covered, which ensures that informed decisions are made for representation change detection, especially when $k=0$. To do that, we use the idea of \textit{random projection} \cite{VR-book:2008}. First, we generate a projection matrix $P$ that projects from $\R^{d-r}$ onto a random $\bar n$-dimensional subspace uniformly distributed in the Grassmann manifold $G_{(d-r),\bar n}$ (which consists of all $\bar n$-dimensional subspaces in $\R^{d-r}$). One can obtain a matrix $Q\in \R^{(d-r)\times \bar n}$ with orthonormal columns that satisfies $P= Q Q^\top$. Then, the $\bar n$ remaining actions are generated by taking the columns out from the matrix $M= B_{\perp} Q$. 

\begin{figure}[t]
	\centering
	\includegraphics[scale=1.6]{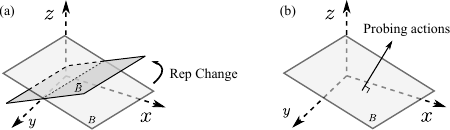}
	\caption{Illustration of environment  change detection. (a) Along with the environment  change, the representation switches from $B$ to $\bar B$. (b) Probing actions are randomly generated in the orthogonal complement of the current representation.}
	\label{Detection}
\end{figure}

Finally, we have completed selecting all the $n_{\rm det}$ probing actions, which are included in the set
\begin{align}
	\CA_{\rm det}:= \{\underbrace{\hat \CA,\dots\hat \CA}_{k}, \lambda_0 [M]_1,\dots,\lambda_0 [M]_{\bar n} \},
\end{align}
where $\hat \CA:=\{\lambda_0 [B_{\perp}]_1, \dots,\lambda_0 [B_{\perp}]_{(d-r)}\}$, and $\lambda_0 $ is a scalar such that all the actions in $\CA_{\rm det}$ is in $\CA$. 

Let $Y_{\rm det}=[y_1,\dots,y_{n_{\rm det}}]^\top$ collect the rewards. We build a confidence interval for $Y_{\rm det}$, which is 
\begin{align*}
	\mathcal C_{\rm det} = \left\{ Y\in \R^{n_{\rm det}}:  \left| \frac{1}{\sqrt{n_{\rm det}} } \|Y\|_2 - 1  \right|\le \xi_{\rm det}\right\},
\end{align*}
where $\xi_{\rm det}$ is the \textit{detection threshold}. For the task $\theta$, if $Y_{\rm det}\in \mathcal C_{n_{\rm det}}$ is observed, we say $\theta$ is not an outlier; if $Y_{\rm det}\notin \mathcal C_{\rm det}$ is observed, we say $\theta$ is an outlier and, subsequently, there is an environment  change. 

\begin{algorithm}[t]
	\caption{Rep Change Detection (RepCD($B$))}
	\label{alg:RepCD}
	\begin{algorithmic}[1]
		\footnotesize
		\State \textbf{Input:} $B \in \R^{d\times r}$, $n_{\rm det}$.
		\State take $n_{\rm det}$ probing actions in $\CA_{\rm det}$, $Y_{\rm det}=[y_1,\dots,y_{n_{\rm det}}]^\top$
		\If{$Y_{\rm det}\notin \mathcal C_{n_{\rm det}}$}
		\State	Rep Change indicator $\mathbb I_{\rm det}=1$
		\EndIf
	\end{algorithmic}
\end{algorithm}

The next lemma shows how to select the detection threshold $\xi_{\rm det}$ and the number of probing actions $n_{\rm det}$ such that an outlier can be detected with high probability. 

\begin{lemma}[Outlier detection: oracle representation]\label{lemma:RepCD}
	Consider two representations $B$ and $\bar B$ and any task $\theta$ satisfying $\theta=\bar B \alpha$ for some $\alpha$. Assume that $\sin \bm{\theta}_r (B, \bar B) = \kappa_1$. Let 
	\begin{align}\label{det_selection}
		n_{\rm det} = {\Bigl \lceil\frac{9(d-r)\log(2S^2 N)}{ \theta_{\min}^2 \lambda_0^2 \kappa_1^2}} \Bigr \rceil,  \xi_{\rm det} = {\sqrt{\frac{\log (2S^2 N)}{4 n_{\rm det}}}}.
	\end{align}
	Then, the task $\theta$ can be detected as an outlier to $B$ by RepCD($B$) in Algorithm~\ref{alg:RepCD} with probability at least
	$
	1-O(\frac{1}{S^2N}).
	$
	
\end{lemma} 

Note that the distance of the two subspaces $B$ and $\bar B$ is measured by the smallest angle $\sin \bm{\theta}_r (B, \bar B)$. From Lemma~\ref{lemma:RepCD}, it can be seen that fewer probing actions (smaller $n_{\rm det}$) are needed to detect a change between two representations with a larger distance.
Notice that, in Lemma~\ref{lemma:RepCD}, we have used the oracle $B$ in RepCD to detect outliers. However, the agent usually has just access to an estimate $\hat B$. The following lemma states that if $\hat B$ is sufficiently accurate, an outlier can still be detected with high probability. 

%\textbf{\textit{Example:}}  Consider the case where there are $S=1000$ tasks and each has dimension $d=100$. The representations $B$ and $\bar B$ are $10$-dimensional ($r=10$) and satisfy $\sin \bm{\theta}_r (B, \bar B) = 0.6$. The tasks satisfy $\theta_{\min}=4$ and $\nu=0.5$. The action set $\CA$ is such that $\lambda_0=4$. From Lemma~~\ref{lemma:RepCD}, one can select $n_{\rm det} =43$ and $\xi_{\rm det} =0.37$.

\begin{lemma}[Outlier detection: estimated representation]\label{lemma:RepCD:estimate}
	Consider the representation change in Lemma~\ref{lemma:RepCD}.  Assume that $\hat B$, which satisfies $\|\hat B^\top B_\perp\|_F\le \varepsilon \le \frac{1}{3} \kappa_1$, is known.  Let $n_{\rm det}$ and $\xi_{\rm det}$ be as in Eq.~\eqref{det_selection}. Then, the task $\theta$ can be detected as an outlier to $B$ by RepCD($\hat B$) in Algorithm~\ref{alg:RepCD} with probability at least
	$
	1-O(\frac{1}{S^2N}).
	$
	
	%	Assume that the agent knows an estimate of the representation $B$. Assume that $\hat B$ that satisfies $\|\hat B^\top B_\perp\|_F\le \frac{1}{3} \kappa_1$ is known. Then, for any $\theta\in \Span(\bar B)$ where $\bar B$ satisfies $\sin \bm{\theta}_r (B, \bar B) = \kappa_1$.  Then, by choosing $n_{\rm det}$ and $\xi_{\rm det}$ as in Eq.~\eqref{det_selection}, the task $\theta$ can be detected as an outlier by RepCD($\hat B$) in Algorithm~\ref{alg:RepCD} with probability at least $ 1-O(\frac{1}{S^2N})$ if $\varepsilon \le \frac{1}{3} \kappa_1$.
\end{lemma}

%pfof}{Lemma~\ref{lemma:RepCD:estimate}}
%	Different from Lemma~\ref{lemma:RepCD}, the true representation $B$ is unknown. Instead, we only know an estimate $\hat B$ that satisfies $\|\hat B^\top_\perp B\|\le \varepsilon$. It can be derived that $\|\hat B^\top_\perp \theta\| =\|B^\top_\perp \bar B \alpha\| \ge (\kappa_1-\varepsilon) \theta_{\min} $. Then, following similar steps as those in the proof of Lemma~\ref{lemma:RepCD}, one can show that the outlier can be detected with probability at least $1-O(\frac{1}{S^2N})$.
%\end{pfof}
%\begin{

\section{The Main Algorithm: CD-RepL}\label{sec:main_sec}
In this section, we provide the main results in this paper.
\subsection{CD-RepL}

We present the main algorithm, i.e., the change-detection representation learning algorithm (CD-RepL). CD-RepL uses the strategies that we have presented in the previous two sections to perform representation learning and to adapt to changing environments. 

CD-RepL proceeds as follows. If a new environment  is detected by RepCD (the first environment is also regarded as a new one), the agent first performs \textit{initial representation exploration}. In this period, $bL$ sequential tasks are played and an initial representation estimate $\hat B$ is constructed. Then, using this $\hat B$, the agent starts to test every task to infer whether there is an environment change using RepCD. If there is no representation change, the agent plays the sequential bandits using SeqRepL in Algorithm~\ref{alg:SeqRepL}. Note that SeqRepL starts from the $(b+1)$th cycle instead of the first one. Meanwhile, $\hat B$ is constantly updated. Once detecting a new environment, the agent restarts the above processes. 

Notice that the initial representation exploration plays an important role in CD-RepL. It provides the agent with a rough but acceptable estimate of the underlying representation such that the agent can avoid false detection with high probability (see Lemma~\ref{lemma:fasle-detection}). Moreover, after the initial exploration SeqRepL does not need to start from the first cycle because of the initial estimate $\hat B$. By starting from $(b+1)$th cycle, it avoids the unnecessary exploration phases in the first $b$ cycles. Let us explain how to select $b$. 

First, we choose the number of probing actions $n_{\rm det}$ and the detection threshold $\xi_{\rm det}$ since the choice of $b$ depends on them. In the previous section, they are chosen in the case of two representations. For the case of $m$ representations, we use the same idea. Let $\kappa:=\min\{\sin \bm{\theta}_r (B_i, B_{i+1}),i=1,\dots,m-1\}$. Then, we let 
\begin{align}\label{det_selection:m}
	n_{\rm det} = {\Bigl \lceil\frac{9(d-r)\log(2S^2 N)}{ \theta_{\min}^2 \lambda_0^2 \kappa^2}\Bigr \rceil},  {\xi_{\rm det} = 2\sqrt{\frac{\log (2S^2 N)}{n_{\rm det}}}}.
\end{align} 
Subsequently, we let 
\begin{align}\label{value:b}
	b = {\Bigl \lceil \frac{9 d r  \theta_{\max}^2 n_{\rm det}}{4 \nu^2  \sqrt{\ell N} (d-r) \log(2 S^2 N)} \Bigr \rceil}. 
\end{align}

The expressions in Eqs.~\eqref{det_selection:m} and \eqref{value:b} are guided by the following underlying ideas: 

(1) Fewer probing actions ($n_{\rm det}$) and a larger detection threshold ($\xi_{\rm det}$) are sufficient to detect representation changes if the distance between representations are larger.

(2) The choice of $b$ needs to balance between the need to explore more tasks such that the initial estimate of $\hat B$ is sufficiently accurate to avoid false detection and the incentive to explore fewer tasks to incur less regret.

\begin{algorithm}[t]
	\caption{CD-RepL}
	\label{alg:main}
	\begin{algorithmic}[1]
		\footnotesize
		\State \textbf{Input:} $b$, $L$ \hspace{5pt} \textbf{Initialize:} $\hat B=I_d$, $\hat P=\bm{0}$,  the new Rep indicator $\mathbb I_{\rm det}=1$
		\If{$\mathbb I_{\rm NewRep}=1$ (i.e., a new environment detected)}
		\State reset $\hat P=\bm{0}$;
		\State play $bL$ tasks using RepE with $N_1=dr\sqrt{\frac{N}{\ell}}$, $\hat P= \hat P+\hat \theta_i\theta_i^\top$;
		\State $\hat B \leftarrow$ top $r$	singular value decomposition of $\hat W=\frac{1}{bL}\hat P$;
		\State set $n_{\rm st}=b+1$.
		\Else 
		\State invoke RepCD, update $\mathbb I_{\rm det}$
		\EndIf 
		\If{$\mathbb I_{\rm det}=0$}
		\State invoke SeqRepL that starts from the $n_{\rm st}$th cycle.
		\EndIf	
	\end{algorithmic}
\end{algorithm}

The following theorem provides an upper bound for CD-RepL, which also justifies our choice of $n_{\rm det}$, $\xi_{\rm det}$, and $b$.  

\begin{theorem}[Upper bound of CD-RepL]\label{upp-bound:CD-RepL}
	Let the agent play the sequential bandits in $\CS$ using the CD-RepL in Algorithm~\ref{alg:main}. Let $L=\Theta(\ell)$,  $n_{\rm det}$ and $\xi_{\rm det}$ be as in Eq.~\eqref{det_selection:m}, and $b$ be as in Eq.~\eqref{value:b}. Then, the regret of the CD-RepL satisfies
	\begin{align}\label{Upper:CD-RepL}
		\BE[R_{T}] = \tilde{O} \Big( &{\underbrace{\sum_{i=1}^{m} \big( {d  r \sqrt{\tau_i N}}+ { \frac{d  r}{\lambda_0^2 \nu^2} \sqrt{\tau_i N}}  +\frac{d \ell}{r} \sqrt{\tau_iN } \big) + S  r \sqrt{N} }_{(a)}} \nonumber\\
		&\hspace{1.8cm}{+\underbrace{S n_{\rm det}+ m b dr\sqrt{N\ell}}_{(b)}  + \underbrace{2m}_{(c)}} \Big).
	\end{align}
\end{theorem}

The regret incurred by CD-RepL can be decomposed into three parts. Term (a) follows from Theorem~\ref{regret:seq:single_B}, which is the regret incurred by the SepRepL for each environment. Term (b) is the regret incurred by the probing actions for each task and the initial representation exploration for each environment. Term (c) is due to unsuccessful detection and false detection. {Similar to Remark~\ref{remark:Th1}, whether CD-SepL  outperforms the existing algorithms that treat bandits independently depends on the relationship between the parameters $d, r,\ell,\lambda_0, \nu, \kappa,\tau,S$ and $m$. To make the dependence more clear, we provide the following corollary. } 

\begin{corollary}
	Assume that $\ell=\Theta(r)$, $1/(\lambda_0^2 \nu^2)=\Theta(1)$, and  $\kappa$ satisfies
	\begin{align}\label{kappa_bound}
		\kappa \ge \max \big\{\min\{ p_1,p_2 \},\min \{p_3,p_4 \} \big\},
	\end{align}
	where
	\begin{align*}
		&p_1= \frac{3S^{\frac{1}{4}} (d-r)^{\frac{1}{2}} \sqrt{\log(2S^2N)}}{ 2d^{\frac{1}{2}} r^{\frac{1}{2}}  N^{\frac{1}{4}} \theta_{\min} \lambda_0},\\
		&p_2=\frac{3(d-r)^{\frac{1}{2}} \sqrt{\log(2S^2N)} }{2 r^{\frac{1}{2}} N^{\frac{1}{4}} \theta_{\min} \lambda_0}\\
		&p_3= \frac{3(d-r)^{\frac{1}{2}}  q^{\frac{1}{2}} m^{\frac{1}{2}} }  {4r^{\frac{1}{2}} N^{\frac{1}{4}} S^{\frac{1}{4}} \nu \theta_{\min} \lambda_0}, \hspace{12pt} p_4=\frac{3 d^{\frac{1}{2}} (d-r)^{\frac{1}{2}} q^{\frac{1}{2}} m^{\frac{1}{2}}}  {4r^{\frac{1}{2}} N^{\frac{1}{4}} S^{\frac{1}{2}} \nu \theta_{\min} \lambda_0}.
	\end{align*}
	Then, the regret of CD-RepL for the sequential tasks in $\CS$ satisfies 
	\begin{align}\label{Cora:CD-RepL}
		\BE [R_{SN}] = \tilde O \Big(\sum \nolimits _{i=1}^{m} {d r \sqrt{ \tau_i N}}+ S r \sqrt{N}  \Big).
	\end{align}
\end{corollary}

The reason that the upper bound in Theorem~\ref{upp-bound:CD-RepL} reduces to the one in Eq.~\eqref{Cora:CD-RepL} is because Terms (b) and (c) in Eq.~\eqref{Upper:CD-RepL} are dominated by Term (a) if $\kappa$ is lower bounded as in Eq.~\eqref{kappa_bound}. Further, observing that $\sum \nolimits _{i=1}^{m} {d r\sqrt{ \tau_i N}} \le dr \sqrt{ S m N}$, the upper bound of the regret in Eq.~\eqref{Cora:CD-RepL} can be rewritten into
$
\tilde O \Big(dr \sqrt{m S N} + S r \sqrt{N}  \Big).
$
Recall that algorithms like UCB, ETC, and PEGE (e.g., see \cite{Dani-Hayes-Kakade:2008,RP-TJN:2010,LY-WY-CX-ZY:2021}) that play the sequential bandits independently have a regret bound $\Theta(Sd\sqrt{N})$. Under the assumption $r\ll d$, CD-RepL outperforms these algorithms considerably if $S/m > r^2 $. In other words, our algorithm is advantageous over the existing ones if: (a) the dimension of linear representation is much smaller than the task dimension; (b) the underlying representation does not change too fast. Notice that if the agent plays the $S$ tasks simultaneously, the regret upper bound can be up to $O(Sd\sqrt{N})$ even if the idea of representation learning is used. This is because the entire set of the $S$ tasks may not share a common representation, although some of its subsets do (this point will be demonstrated in Fig.~\ref{synthetic} in Section~\ref{sec:Experiments}). 

\begin{remark}
	We remark that the assumption of $\kappa$ in Eq.~\eqref{kappa_bound} is mild. The right-hand side of Eq.~\eqref{kappa_bound} becomes very small if $d\ll N$.
	
	%	Let us illustrate this point by an example. Consider the case where there are $S=1000$ tasks, each of dimension $d=100$ and  played for $N=10000$ rounds. The representations are $10$-dimensional ($r=10$). The tasks satisfy $\theta_{\min}=4$ and $\nu=1$. The action set $\CA$ is such that $\lambda_0=4$. In this case, $\kappa$ only needs to satisfy $\kappa \ge 0.077$. Roughly speaking, the smallest angle distance between any two consecutive representations just needs to exceed $0.025\pi$. 
\end{remark}

\subsection{Analysis of Theorem~\ref{upp-bound:CD-RepL}}
Let us provide an instrumental result first, whose proof is in Appendix~\ref{pf:lm5}. 

\begin{lemma}[Probability of a false detection]\label{lemma:fasle-detection}
	Consider the case where there is only one environment (i.e, $m=1$) in the task sequence $\{\theta_1,\theta_2,\dots,\theta_S\}$, and the underlying representation is $B\in \R^{d\times r}$. Let the agent play this sequence of tasks using CD-RepL. Let $\hat B$ be the estimated representation after the initial sample phase of $bL$ tasks, the probability that a task $\theta$ is detected by RepCD($\hat B$) as an outlier, denoted by  $\Pr [Y_{\rm det} \notin \mathcal C_{\rm det} |\theta=B\alpha]$, is less than $O(\frac{1}{S^2N})$. 
\end{lemma}

\begin{pfof}{Theorem~\ref{upp-bound:CD-RepL}}	
	Recall that $\tau_i$ sequential tasks are taken from each set $\CS_i$. Therefore, the representations change after the $(\sum_{i=1}^{p-1}\tau_i)$-th task is played, where $p=2,\dots,m$. For the simplicity of nation, denote $v_i=(\sum_{j=1}^{i}\tau_j+1), i=1,\dots,m-1$ as the instants when the environment switches happen. Let $\mu_i$ be the detection time of the $i$th switch (i.e., the $\mu_i$th task is detected as an outlier to the $i$th representation).  Therefore, the event $D_i=\{\mu_i=v_i\}$ is a good event, which describes the situations that the $i$th representation switch is detected immediately after it happens. The event $G_i=\{\mu_i>v_i\}$ denotes a late or an unsuccessful detection. The event $F_i=\{\mu_i<v_i\}$ denotes a false detection, which describes the situation that an alarm is triggered when there is no representation switch.
	
	First, we consider the case with two representations (i.e., 1 representation switch), and the sequential tasks $\underbrace{\{\theta_1,\dots,\theta_{\tau_1}}_{B_1},\underbrace{\theta_{\tau_1+1},\dots,\theta_{\tau_1+\tau_2}}_{B_2}\}$ are played using CD-RepL in Algorithm~\ref{alg:main}. The regret of CD-RepL given $D_1$ satisfies
	\begin{align}\label{ineq:D1}
		{\BE [R_{(\tau_1+\tau_2) N}|D_1]\lesssim} &{\underbrace{ \sum_{i=1}^{2} {d  r \sqrt{\tau_i N}}+{\frac{d  r}{\lambda_0^2 \nu^2} \sqrt{\tau_i N}}+  \frac{d \ell}{r} \sqrt{N \tau_i} + \tau_i  r \sqrt{N} }_{(a)}} \nonumber \\
		&{+\underbrace{(\tau_1+\tau_2)n_{\rm det}}_{(b)} +\underbrace{2 b dr\sqrt{N\ell} }_{(c)}}.
	\end{align}
	Note that term (a) is the regret incurred by SeqRepL in individual subsequence, which follows from Theorem~\ref{regret:seq:single_B}; term (b) is incurred by the probing actions for each task; and term (c) results from the initial representation exploration. 
	
	If the event $G_1$ happens, the regret of CD-RepL satisfies
	\begin{align}\label{ineq:G1}
		{\BE [R_{(\tau_1+\tau_2) N}| G_1]\lesssim}&{\underbrace{( {d  r \sqrt{\tau_1 N}}+{\frac{d  r}{\lambda_0^2 \nu^2} \sqrt{\tau_1 N}}+  \frac{d \ell}{r} \sqrt{N \tau_1} + \tau_1  r \sqrt{N}) }_{(i)}}\nonumber\\
		&{\underbrace{(\tau_1+\tau_2)n_{\rm det} + b dr\sqrt{N \ell} }_{(ii)} + \underbrace{\tau_2 N}_{(iii)}}.
	\end{align}
	Here, the terms (i) and (ii) are the regret in the first subsequence, which follows from Eq.~\eqref{ineq:D1}. If the representation switch is detected after $v_i$ or not detected at all, the regret for the second subsequence would be bounded by $O(\tau_2 N$), which results in the term (iii). This is intuitive since there are $\tau_2$ tasks in the second subsequence and the regret of each task is bounded by $O(N)$. 
	
	If the event $F_1$ happens, the regret of CD-RepL satisfies	
	\begin{align}\label{ineq:F1}
		\BE [R_{(\tau_1+\tau_2) N}|F_1]\lesssim (\tau_1+\tau_2) N.
	\end{align}
	From Lemma~\ref{lemma:RepCD}, it can be derived that $\Pr[D_1]\ge 1-O(\frac{1}{S^2N})$. The event $G_1$ means that $D_1$ does not happen. It can be calculated that $\Pr[G_1]\le O(\frac{1}{S^2 N})$. The event $F_1$ means that there is at least $1$ false detection in the first subsequence. From Lemma~\ref{lemma:fasle-detection}, we know that the probability that a task is detected as an outlier falsely is less than $O(\frac{1}{S^2N})$. Then, the probability of $F_1$ can be calculated as
		\begin{align*}
			\Pr[F_1] \le  1- (1-O(\frac{1}{S^2 N}))^{\tau_1} \le O(\frac{\tau_1}{S^2N}). 
	\end{align*}
	
	Putting together Eqs.~\eqref{ineq:D1}--\eqref{ineq:F1} and using the law of total expectation, it can be computed that the regret of CD-RepL satisfies
	\begin{align}\label{Upper:m=2}
		{\BE [R_{(\tau_1+\tau_2) N}]}  &{\lesssim \sum_{i=1}^{2} ( {d  r \sqrt{\tau_i N}}+{\frac{d  r}{\lambda_0^2 \nu^2} \sqrt{\tau_i N}}+  \frac{d \ell}{r} \sqrt{N \tau_i} + \tau_i  r \sqrt{N} )} \nonumber\\
		&+(\tau_1+\tau_2)n_{\rm det}+ 2 b dr\sqrt{N\ell}  + 2, 
	\end{align}
	where the last term on the right-hand side follows from $\tau_2 N \cdot O(\frac{1}{S^2n})+ (\tau_1+\tau_2) N \cdot O(\frac{\tau_1}{S^2N})\le O(2)$. 	
	
	The upper bound in Eq.~\eqref{Upper:m=2} can be generalized to the case where there are $m-1$ representation switches. In this case, the upper bound of the regret of CD-RepL becomes 	
	\begin{align*}
		{\BE [R_{S N}] \lesssim}&  {\sum_{i=1}^{m} ( {d  r \sqrt{\tau_i N}}+{ \frac{d  r}{\lambda_0^2 \nu^2} \sqrt{\tau_i N}}+  \frac{d \ell}{r} \sqrt{N \tau_i} + \tau_i  r \sqrt{N} )}\\
		&{+\sum_{i=1}^{m} \tau_i n_{\rm det}+ m b dr\sqrt{N\ell}  + 2(m-1)}\\
		\le& {\sum_{i=1}^{m} \big( {d  r \sqrt{\tau_i N}}+{ \frac{d  r}{\lambda_0^2 \nu^2} \sqrt{\tau_i N}}+  \frac{d \ell}{r} \sqrt{\tau_iN } \big) + S  r \sqrt{N} }\\
		&{+S n_{\rm det}+ m b dr\sqrt{N\ell}  + 2m}.
	\end{align*}
	which completes the proof.
\end{pfof}

\section{Illustrative Examples}\label{sec:Experiments}

We perform some  experiments to validate our theoretical results and demonstrate the efficacy of our algorithm.

%\textbf{LastFM} This dataset is extracted from the music streaming service Last.fm. It contains 1892 users, 17632 artists, and a listening count of each user-artist pair. We first remove the artists that have fewer than 40 listeners and the users who listened fewer than 20 artists, and obtain a matrix $M=[m_{ij}]$ of size $411\times 1107$.  According to this matrix, we generate payoffs: $r_{ij}=1$ if $m_{ij}\le 120$, $r_{ij}=2$ if $120< m_{ij}\le 250$, $r_{ij}=3$ if $250 < m_{ij}\le 500$, $r_{ij}=4$ if $500< m_{ij} \le 120$, $r_{ij}=5$ if $m_{ij}> 1200$. 
%
%We generate payoffs using the information in this dataset. 

\textbf{Synthetic Data.} We first synthesize a set of data to demonstrate our algorithm. Specifically, we consider a series of $1600$ bandit tasks of dimension 20. There are four segments in this sequence, and each has 400 tasks. In each segment, there is a representation $B_i\in \R^{20\times 3}$. The parameters in Assumption~\ref{Assump:diversity} are $\ell=3$ and $\nu=0.01$. The action set is the unit ball defined by $\CA=\{x\in\R^{20}:\|x\|\le 1\}$. The noise in the reward-generating function is assumed to be Gaussian $\mathcal N(0,0.3)$. Each task is played for $2000$ rounds.

\begin{table}[t!] 
	\caption{Baselines} % title name of the table  
	\centering % centering table  
	\begin{tabular}{l c} % creating 10 columns   
		\hline   
		Algorithms &Description
		\\ [0.2ex]  
		\hline  \hline 
		% Entering 1st row  
		& {\scriptsize play bandits independently} \\[-1.5ex]
		\raisebox{1ex}{Standard} & {\scriptsize (e.g., ETC \cite{AYY-AA-SC:2009},  PEGE \cite{RP-TJN:2010} and UCB \cite{Dani-Hayes-Kakade:2008,LY-WY-CX-ZY:2021})}\\
		\hline
		& {\scriptsize the single representation for all bandits is known} \\[-1.5ex]
		\raisebox{1ex}{Semi-oracle} & {\scriptsize (i.e.,  $B$ such that satisfies $\theta_i=B\alpha_i$ for all $i$ is known\tablefootnote{Such $B$ always has no smaller dimension that those of within-environment representations, and it may not exist if $\theta_i$'s span the entire $\R^d$. Note that existing algorithms (e.g., \cite{Yang-Hu-Lee:2021,HJ-CX-JC-LL-WL:2021}) that play bandits simultaneously and exploit representation learning cannot outperform the semi-oracle algorithm since they need to estimate the representation. })}\\
		\hline
		{Non-adaptive} & {\scriptsize disabled environment change detection}\\
		\hline
		{Oracle} & {\scriptsize representations and change times are known}\\
		\hline
	\end{tabular}  \label{table}
\end{table}    

\begin{figure}
	\centering		
	%	\subfigure
	{\includegraphics[scale=1.5]{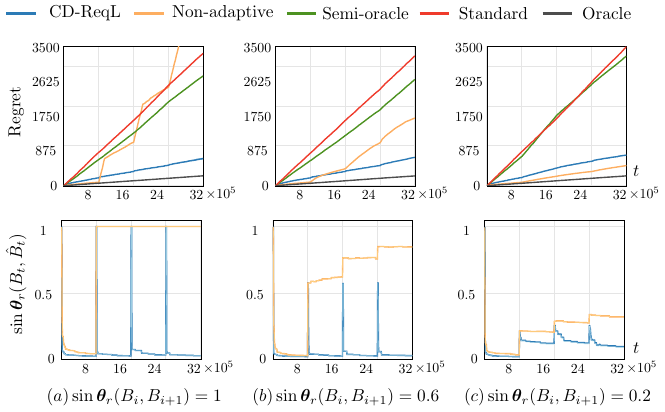}}
	\caption{Performance comparison between different algorithms using synthetic data. Upper panels: regret; lower panels: distance between $\hat B$ to the changing true representation $B$. In (a), (b), and (c), environment changes involving different representation distances are considered. Shaded areas contain 10 random realizations of trials.}
	\label{synthetic}
\end{figure}

\begin{figure}
	\centering		
	%	\subfigure
	{\includegraphics[scale=1.5]{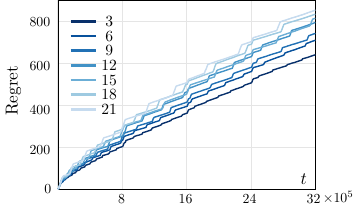}}
	\caption{Performance of CD-SepL when $L$ is set to different values. The parameters are the same as in Fig.~\ref{synthetic}~(b).}
	\label{Diff_L}
\end{figure}

We compare our algorithm (CD-RepL) with four other algorithms in Table~\ref{table}. We consider several situations where environment changes involve different representation distances (which is measured by $\sin\bm{\theta}_r(B_i,B_{i+1}),i=1,2,3$). 

The Oracle outperforms CD-RepL as expected since the latter needs to pay the cost of learning the representations and detecting the environment changes.

From the upper panels in Fig.~\ref{synthetic}, CD-RepL always outperforms Standard, consistent with our theoretical results. Compared with Semi-oracle, CD-SeqL outperforms the existing algorithms that learn the single representation shared by all the bandits. The reason is that this representation may not exist or have a high dimension even if the subsets of tasks have low-dimensional representations. 

Further, we disable the environment change detection of our algorithm (i.e., Non-adaptive) and compare it with CD-RepL. CD-RepL is particularly advantageous over Non-adaptive when representations change drastically (see Fig.~\ref{synthetic}~(a)); the advantage decreases when it comes to more subtle changes (see Fig.~\ref{synthetic}~(b)). For sufficiently small changes, Non-adaptive can even perform better (see Fig.~\ref{synthetic}~(c)). This is because the price of detecting the changes and re-learning each representation may overwhelm the potential benefits of transferring the learned representations. However, CD-RepL have a much more stable performance in all situations.  The lower panels in Fig.~\ref{synthetic} illustrate that CD-SepL can detect environment changes in different situations involving drastic or subtle representation changes, which is the key to endow CD-SepL with the adaptability. By contrast, Non-adaptive can no longer learn the true representations accurately after the first environment change. 

%consider a system that recommends movies to users in sequence. At each step, this system recommends a movie to a user. If the user likes this movie, they have a higher probability to click and watch it. Therefore, recommending this movie has a higher reward. Each users' preference is determined by a $10$-dimensional vector ($\theta_i\in \R^d, d=10$) that describes their features. The information of each movie is also described by a $10$-dimensional vector. The reward that recommending the movie associated with $x_t$ to the user $\theta_i$ is generated by $y_t=x_t^\top \theta_i +\eta_t$, where $\eta_t$ models uncertainty.  There are $S=1000$ users, and the system recommends $N=200$ times to each user. The users form $3$ consecutive segments, and in each segment there is a $3$-dimensional representation  ($r=3$) for the user features. The lengths of the segments are 200, 400, and 400, respectively.  The goal of the recommender system is to maximize the cumulative reward over the course of $SN$ steps.

{ Recall that our theoretical results rely on fact that the order of $\ell$ in Assumption~\ref{Assump:diversity}  is known since the number of tasks in each representation cycle is set to be $L=\Theta(\ell)$. Yet, this assumption is not required in practice. As shown in Fig.~\ref{Diff_L}, CD-RepL has similar performance for a wide range of $L$. This implies that, even if $\ell$ is unknown, one can always choose an $L$ that is likely larger than $\ell$ without compromising the performance much compared to the case of letting $L=\ell$.}

% why CD-SeqL performs well. First, it invokes SeqRepL (see Algorithm~\ref{alg:SeqRepL}) within each {\rd environment}. It can be seen that the learned representation $\hat B$ becomes progressively more accurate as more tasks within each {\rd environment} are played. Note that this is the key feature that enables SeqRepL to perform well in {\rd environments} with different duration. Second, CD-SeqL can quickly detect {\rd environment} changes and adapt to new environments.  After disabling the algorithm from detecting {\rd environment} changes, it loses adaptability, and can 

% \begin{figure}
	% 	\centering
	% 	\subfigure[]{\includegraphics[scale=0.73]{./img/recomm/comparison.pdf}}
	% 	\subfigure[]{\includegraphics[scale=0.73]{./img/recomm/subspace_dist.pdf}}
	% 	\caption{(a) Performance comparison between different algorithms using synthetic data. (b) distance between $\hat B$ to the changing true representation $B$. Shaded areas contain 10 realizations.}
	% 	\label{synthetic}
	% \end{figure}

\textbf{LastFM.} We use this dataset to demonstrate that our algorithm can be used to design an adaptive recommendation system as depicted in Fig.~\ref{conceptual}~(b). This dataset is extracted from the music streaming service Last.fm. It contains 1892 users, 17632 artists, and a listening count of user-artist pairs. We first remove the artists that have fewer than 40 listeners and the users who listened fewer than 10 artists, and obtain a matrix $M=[m_{ij}]$ of size $411\times 1565$ with each row representing an artists and each column a user.  To generate arms and users, we use the non-negative matrix factorization for $M$ and keep the first $20$ latent features. In other words, $M\approx AU$, where $A\in \R^{411\times 20}$ and $U\in \R^{20\times 1565}$ are non-negative and describe the features of the artists and users, respectively.  From $U$, we select 3 groups of users that approximately lie in distinct subspaces, consisting $11$, $6$, and $6$ users, respectively. These users form a series of bandits $\theta_1,\dots,\theta_{23}$ that have different 2-dimensional representations. We then recommend music items to these users for 200 times from the action set $\CA$ that is composed of the row vectors of $A$. The reward is generated by $y_t=x_t^\top \theta_i+\eta_t$, where $x_t \in \CA$ and $\eta_t$ is Gaussian noise $\mathcal N(0,0.2)$. In Fig.~\ref{LastFM} it can be observed that, by learning and exploiting the representation shared by users and detecting environment changes our algorithm outperforms the existing ones that treat bandits independently. 

%According to this matrix, we generate payoffs: $r_{ij}=1$ if $m_{ij}\le 120$, $r_{ij}=2$ if $120< m_{ij}\le 250$, $r_{ij}=3$ if $250 < m_{ij}\le 500$, $r_{ij}=4$ if $500< m_{ij} \le 120$, $r_{ij}=5$ if $m_{ij}> 1200$. 

\begin{figure}[t]
	\centering
	\includegraphics[scale=0.8]{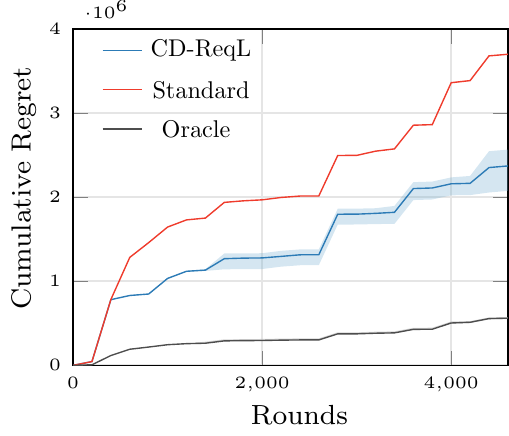}
	\caption{Performance comparison between different algorithms as recommendation systems using LastFM data. Shaded areas contain 10 random realizations of trials.}
	\label{LastFM}
\end{figure}

\textbf{Wisconsin Card Sorting Task (WCST).} WCST, see Fig.~\ref{conceptual}~(b), is typically utilized to assess human abstraction and shift of contexts \cite{GDA-BW:1948}. Participants need to match a series of stimulus cards to one of the four cards on the table based on a sorting rule. For the stimulus card in Fig.~\ref{conceptual}~(b), if the rule is color, the correct sorting action is the third card. The participants only receive feedback about whether their actions are correct. For convenience, we assume that they receive reward 1 for a correct action, and 0 otherwise. Participants do not know the current sorting rule, thus need to infer it by trial and error. The sorting rule changes every now and then, which makes the task challenging.

% Participants need to learn the changes and adjust their strategy. 

WCST can actually be described by a bandit problem with environment changes, where representations define the sorting rules. Specifically, we use a matrix of size $4 \times 3$, $A=[a_1,a_2,a_3]$, to describe each card. Here, the vectors $a_1,a_2$ and $a_3$ define the number, color, and shape, respectively, and they take value from the $4$-dimensional standard basis $\{e_1,e_2,e_3,e_4\}$. The matrix  $A=[e_i,e_j,e_m]$ describes a card that has the $i$th number, the $j$th color, and the $m$th shape of the 4 cards on the table. For instance, the stimulus card in Fig.~\ref{conceptual}~(b) can be described by the matrix $A=[e_2,e_3,e_1]$.  Further, we use a vector $B$ to describe the sorting rule, which takes value from the  basis $\{b_1,b_2,b_3\}$ of $\R^3$, representing the sorting rule is number, color, and shape, respectively.

%since it has the second number (2), the third color (green), and the first shape (circle). 

%\begin{figure}[t]
%	\centering
%	\includegraphics[scale=0.38]{./img/cards.pdf}
%	\caption{WCST: Participants are given a series of stimulus cards and they need to associate each of them with one of the four cards on the table based on changing sorting rules --- number, color, and shape. For instance, if the rule is color, the correct sorting action is the third card for the above stimulus card. The participants receive reward 1 for a correct action, and 0 otherwise.}
%	\label{WCST}
%\end{figure}

\begin{figure}[t]
	\centering
	\includegraphics[scale=0.85]{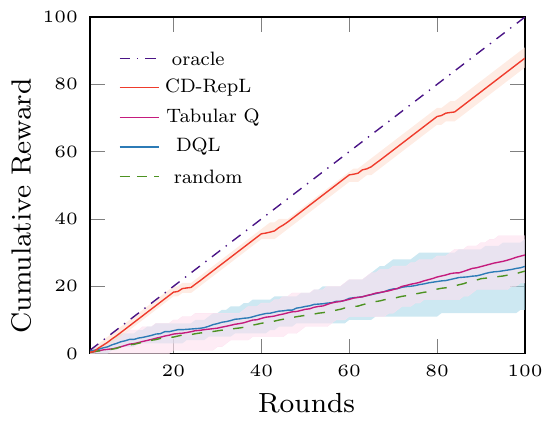}
	\caption{Performance comparison for different algorithms in WCST, where the sorting rule changes after ever 20 rounds. Shaded areas contain 10 realizations.}
	\label{WSCT:comparison}
\end{figure}

As a consequence, the reward of  WCST is generated by $y_t=x^\top_t \theta_t$ with $\theta_t=A_t B_{\sigma(t)}$, where $x_t$ is the action that takes value from the $4$-dimensional standard basis $\{e_1,e_2,e_3,e_4\}$, $A_t$ is the card at round $t$. Notice that $B_{\sigma(t)}$ that describes the sorting rule can be regarded as the time-varying representation in the reward function. 

For standard reinforcement learning algorithms, such as tabular-Q-learning and deep-Q-learning, WCST is challenging. There are in total $4^3$ possible stimulus cards (4 colors, 4 numbers, 4 shapes), and for each stimulus card, there are 4 possible categories. To find the best policy, the standard tabular-Q-learning needs many samples to construct the Q table for a single sorting rule.  It is then impossible for the standard Q learning algorithm to find the optimal policy if the sorting rule changes in a few trials (in Fig.~\ref{WSCT:comparison} it changes every 20 rounds). Being unaware of rule changes results in an even worse performance. The deep-Q-learning algorithm\footnote{Here, we formalize each input
	state by a 3-dimensional vector (shape, number, color)$^\top\in\{1,2,3,4\}^3$. The result in Fig.~\ref{WSCT:comparison} considers a three layer neural network with 3, 12, and 4 nodes in the input, hidden,
	and output layers, respectively. Deeper	or wider structures were also considered, but similar performances were obtained.} does not perform better since, similarly, it always needs a large number of samples to train the weights in the neural network to find the optimal policy. It can be seen from Fig.~\ref{WSCT:comparison} that these two algorithms perform barely better than the one that takes a random action at each round.

However, by describing the WCST as a linear bandit model, we find that the problem can reduce to learning the representation $B_{\sigma(t)}$, as the correct action can be computed as $x^*_t=A_tB_t$. The problem then reduces to learn the underlying representation $B_\sigma$, a task that is much easier than constructing the Q table or training the weights in a Deep-Q network. Remarkably, one does not even need to learn individual $\theta_t$ to construct $B_\sigma$. Instead, $B_\sigma$ can be recovered as
$
B_\sigma=(\sum\nolimits_{t=1}^{k}A_t^\top x_t x_t^\top A_t)^{-1} \sum\nolimits_{t=1}^{k}A_t^\top x_t y_t 
$
immediately after $\sum_{t=1}^{k}A_t^\top x_t x_t^\top A_t$ becomes invertible. This indicates that our idea in this paper can apply to more general situations. As shown in Fig.~\ref{WSCT:comparison}, our algorithm significantly outperforms the other two, approaching the oracle that makes correct choice at every round. This experiment suggests that the ability to learn representations and shift attention  \cite{RA-SYS-NY:2021,NM-MR-XC-PA:2018} to adapt to environment changes facilitates efficient learning.

\section{Concluding Remarks}\label{sec:conclusion}

In this paper, we exploit representation learning for decision-making in non-stationary environments using the framework of multi-task sequential linear bandits. We propose an efficient decision-making algorithm that learns and transfers representations online. Employing a representation-change-detection strategy, our algorithm also has the flexibility to adapt to new environments. We further obtain an upper bound for the algorithm, analytically showing that it significantly outperforms the existing ones that treat tasks independently. Moreover, we perform some experiments using synthetic data to demonstrate our theoretical results. Using the LastFM data, we show that our algorithm can be applied to designing adaptive recommendation systems. In the Wisconsin Card Sorting Task, experimental results show that our algorithm considerably outperforms some classic reinforcement learning algorithms.

Directions of future work include nonlinear representation learning, representation-based clustering, and task-tailored representation generation from experience in bandit and reinforcement learning problems.  

%\begin{center}
%	\textbf{ Supplementary Materials}
%\end{center}

\section*{Appendix} 
\appendix \label{proof:seqRepL}
%\subsection{Proof of Theorem~\ref{dist:B}} \label{proof:seqRepL}

\subsubsection{Proof of Lemma~\ref{regret:RepL}}\label{pf:lm1}

\begin{proof}
	Without loss of generality, we assume that $N_1$ is a multiple of $d$. Following similar steps as those in Lemma~3.4 of \cite{RP-TJN:2010}, we can obtain that after $N_1$ steps of exploration 
	\begin{align*}
		\BE [\|\hat \theta -\theta\|^2] \le \frac{d^2 }{ N_1}.
	\end{align*}
	From \cite{RP-TJN:2010}, it holds that $\max_{x\in \mathcal A} x^\top \theta -\max_{x\in \mathcal A}x^\top \hat \theta \le J{\|\theta-\hat \theta\|^2}/{\|\theta\|}$, where $J$ is a constant that exists since the action set $\CA$ is  an ellipsoid.  Since $\|\theta\|\ge \theta_{\min}$, it follows that
	\begin{align}\label{exp_inequ}
		\BE \big[ \max_{x\in \mathcal A} x^\top \theta -\max_{x\in \mathcal A}x^\top \hat \theta \big] \le J \frac{\BE \|\theta-\hat \theta\|^2}{\theta_{\min}} \le J\frac{d^2}{N_1 \theta_{\min}}.
	\end{align} 
	Further, at the exploration phase it holds  with $g(\theta):=\argmax_{x \in \mathcal A} x^\top \theta$ that
	\begin{align*}
		g^\top (\theta) \theta - x_t^\top \theta &\le \max_{x\in \mathcal A} x^\top \theta -\max_{x\in \mathcal A}x^\top (-\theta) \le 2J\theta_{\max}.
	\end{align*}
	Therefore, the total regret in $N$ steps satisfies
	\begin{align*}
		\BE [R_N] &\le  2J \theta_{\max} N_1 + (N-N_1) J\frac{d^2}{N_1 \theta_{\min}}
		\\& = O(N_1+\frac{N}{N_1}d^2),
	\end{align*}
	which completes the proof.
\end{proof}

\subsubsection{Proof of Lemma~\ref{transfer:single}}\label{pf:lm2}

\begin{proof}
	Without loss of generality, we assume  that $N_2$ is a multiple of $ r$. 	Recall that, at the end of the exploration phase, $\hat \alpha$ is computed by $\hat \alpha = (\hat B^\top X_TX_T^\top \hat B)^{-1} \hat B^\top X_T Y_T$. 	Since $x_t$ repeatedly takes actions from $a'_1,\dots, a'_{ r}$, it holds that $X_T X_T^\top = {N_2} AA^\top/r$ with $A= [a'_1,\dots,a'_{ r}]$. Therefore, we have $\hat \alpha = \big( {N_2}\hat B^\top AA^\top \hat B /r \big)^{-1} \hat B^\top X_T Y_T.$
	%	\begin{align*}
		%		\hat \alpha = \big( \frac{N_2}{ r}\hat B^\top AA^\top \hat B \big)^{-1} \hat B^\top X_T Y_T.
		%	\end{align*}
	Since $a_i'=\lambda_0 [\hat B]_i$, it holds that $A=\lambda_0 \hat B$, which implies that $\big({N_2}\hat B^\top AA^\top \hat B /r \big)^{-1}=\frac{r}{\lambda_0^2 N_2} I_r$. Consequently, 
	\begin{align*}
		\hat \alpha = \frac{r}{\lambda_0^2 N_2} \hat B^\top X_T Y_T.
	\end{align*}	
	As $Y_T = X_T^\top B \alpha +\eta$ with $\eta=[\eta_1,\dots,\eta_{N_2}]^\top$, we have 
	\begin{align*}
		\hat \alpha &= \frac{r}{\lambda_0^2 N_2} \hat B^\top X_T (X_T^\top B \alpha+\eta)\\
		& =  \frac{r}{\lambda_0^2 N_2} \frac{N_2}{ r}\hat B^\top AA^\top  B \alpha +   \frac{r}{\lambda_0^2 N_2}  \hat B^\top X_T \eta \\
		& = \hat B^\top B \alpha +  \frac{r}{\lambda_0^2 N_2}  \hat B^\top X_T \eta .
	\end{align*}
	As $\hat \theta=\hat B \hat \alpha$ and $\theta= B \alpha$, it follows that
	\begin{align*}
		\hat \theta - \theta=\hat B \hat \alpha-B \alpha= \underbrace{\hat B \hat B^\top B \alpha-  B \alpha }_{s_1}+ \underbrace{\frac{r}{\lambda_0^2 N_2}  \hat B \hat B^\top X_T \eta}_{s_2}.
	\end{align*}
	
	Next, we evaluate $\BE \|\hat \theta - \theta\|^2$. Since $\BE s_1^\top s_2=0$, it holds that $\BE \|\hat \theta - \theta\|^2\le \BE \|s_1\|^2+\BE \|s_2\|^2$.
	
	Observe that $I_d= \hat B \hat B^\top +\hat B_\perp \hat B^\top_\perp$. Thus, we have
	\begin{align*}
		\BE \|s_1\|^2=  \|(I_d -\hat B_\perp \hat B^\top_\perp) B \alpha-  B \alpha\|^2 =\|\hat B_\perp \hat B^\top_\perp B \alpha\|^2 .
	\end{align*}
	Because $\|\hat B^\top_\perp B\|_F \le \varepsilon$, we have 
	\begin{align}\label{s_1}
		\BE \|s_1\|^2  \le \|\hat B_\perp \|^2 \cdot \|\hat B^\top_\perp B\|^2_F \cdot \|\alpha\|^2\le \mu \varepsilon^2,
	\end{align}
	where $\mu>0$, which is such that $\|\alpha\|^2\le \mu$, exists since $\theta=B\alpha$ satisfies $\theta_{\min} \le \theta \le \theta_{\max}$. 
	
	Now, we evaluate  $\BE \|s_2\|^2$, which satisfies
	\begin{align*}
		\BE \|s_2\|^2 &=\frac{r^2}{\lambda_0^4 N_2^2}  \BE (\hat B \hat B^\top X_T \eta)^\top  \hat B \hat B^\top X_T \eta\\
		&= \frac{r^2}{\lambda_0^4 N_2^2} \sum_{t=1}^{N_2} x_t^\top \hat B \hat B^\top x_t \BE \eta_t^2.
	\end{align*}
	Since $\eta_t$ is sub-Gaussian with variance proxy variable $1$, we have 
	\begin{align}\label{s_2}
		\BE \|s_2\|^2 \le \frac{r^2}{\lambda_0^4 N_2^2} \sum_{t=1}^{N_2} x_t^\top \hat B \hat B^\top x_t = \frac{r^2}{\lambda_0^2 N_2}. %=\frac{r^2}{\lambda_0^4 N_2^2} N_2 \lambda_0^2
	\end{align}
	Putting Eqs.~\eqref{s_1} and \eqref{s_2} together, we have $\BE [ \|\hat \theta -\theta\|^2 ] \le \frac{ r^2 }{\lambda_0^2 N_2} + \mu \varepsilon^2$. 	
	Similar to \eqref{exp_inequ}, one can derive that
	\begin{align}\label{Expectation:1}
		\BE[\max_{x\in \mathcal A} x^\top \theta -\max_{x\in \mathcal A}x^\top \hat \theta]
		\le J\frac{ r^2}{ \theta_{\min}\lambda_0^2 N_2} +J \frac{1}{\theta_{\min}}  \mu \varepsilon^2.
	\end{align} 
	For the commitment phase, there are $N-N_2$ steps. Thus, the overall regret satisfies
	\begin{align*}
		\BE [R_N] \le 2J\theta_{\max} N_2 + (N-N_2) \BE \big( \max_{x\in \mathcal A} x^\top \theta -\max_{x\in \mathcal A}x^\top \hat \theta \big).
	\end{align*}	
	Substituting \eqref{Expectation:1} into the right side yields $\BE [R_N] =  O( N_2+ \frac{N }{N_2}r^2+ N \varepsilon^2)$, which completes the proof. 
\end{proof}

\subsubsection{Proof of Theorem~\ref{dist:B}}\label{pf:Th2}

\begin{lemma}[Matrix Bernstein's inequality \cite{VR-book:2008}]\label{lemma:sym-matr}
	Let $X_1,X_2,\dots,X_k$ be independent zero-mean $d \times d$ symmetric random matrices so that there exists $M>0$ such that $\|X_i\|\le M$ almost surely for all $i=1,2,\dots,k$. Then, for any $t\ge 0$, it holds that
	$
	\Pr \left[ \left\|\sum_{i=1}^{k}X_i \right\|\ge t\right] \le 2 d \exp \left(\frac{-2 t^2}{\sigma^2+Mt/3} \right),
	$
	where $\sigma^2=\left\| \sum_{i=1}^{k}\BE X_i\right\|$. 
\end{lemma}

\begin{proof}
	Denote the $nL$ tasks that  are played using RepE as $\theta_1,\theta_2,\dots, \theta_{nL}$, and let $k :=nL$.  Then, $\hat W$ becomes  $\frac{1}{k}\sum_{i=1}^{k}\hat \theta_i\hat\theta_i^\top$. Let $W=\frac{1}{k}\sum_{i=1}^{k} \theta_i \theta_i^\top$ be the true counterpart of $\hat W$. The proof is constructed in two steps. In Step 1, we use Lemma~\ref{lemma:sym-matr} to estimate $\|\hat W -(W+D)\|$ with $D$ being a scaled identity matrix; in Step 2, we use the Davis-Kahan $\sin \bm{\theta}$ Theorem \cite{BR-book:2013} to evaluate the distance between the top-$r$ singular values of $\hat W$ and $W+D$, which is the distance between $\hat B$ and the true $B$ (notice that $W+D$ and $W$ share the same singular vectors). 
	
	%	For each bandit $\theta_i, i= 1,\dots, k$, the model  is described by $y_t= x_t^\top \theta_i+\eta_t$. 
	
	\textbf{Step 1.} Let $\eta :=[\eta_1,\eta_2,\dots,\eta_{N_1}]^\top$, and it holds that $Y_E=X_E^\top \theta_i+\eta$. It follows that 
	\begin{align*}
		\hat \theta_i = (X_E X_E^\top)^{-1} X_E(X_E^\top \theta_i+\eta) = \theta_i+ (X_E X_E^\top)^{-1} X_E\eta. 
	\end{align*}
	Some algebraic computations yield 
	\begin{align*}
		\hat \theta_i \hat \theta_i^\top = &\theta_i \theta_i^\top +\theta_i \eta^\top X_E^\top (X_E X_E^\top)^{-1}+ (X_E X_E^\top)^{-1} X_E\eta\theta_i^\top
		\\& +(X_E X_E^\top)^{-1} X_E \eta\eta^\top X_E ^\top (X_E X_E ^\top)^{-1}.
	\end{align*}
	Since $\eta_i$ are independent zero mean 1-sub-Gaussian random variables, the expectation of $\hat \theta_i \hat \theta_i^\top$ can be computed as
	\begin{align*}
		\BE \hat \theta_i \hat \theta_i^\top &=\theta_i \theta_i^\top +  (X_EX_E^\top)^{-1} X_E \BE \eta\eta^\top X_E^\top (X_E X_E^\top)^{-1}\\
		&= \theta_i \theta_i^\top + (X_E X_E^\top)^{-1}. 
	\end{align*}
	Since $[a_1,a_2,\dots,a_d]$ be the standard basis of $\R^d$, it holds that $\sum_{i=1}^{d} a_ia_i^\top = I_d$. Without loss of generality, we consider $N_1$ as a multiple of $d$, then it follows that $X_E X_E^\top =  \frac{N_1}{d} \lambda_0^2 I_d $. Therefore, we have $(X_E X_E^\top)^{-1}= \frac{d}{\lambda_0^2 N_1}I_d$. Denote $D:= \frac{d}{\lambda_0^2 N_1}I_d$, then it follows that 
	\begin{align*}
		\BE \hat \theta_i \hat \theta_i^\top = \theta_i \theta_i^\top + D,
	\end{align*}
	and 
	\begin{align}
		{\hat \theta_i \hat \theta_i^\top = \theta_i \theta_i^\top +D\underbrace{\left( \theta_i \eta^\top X_E^\top +  X_E\eta\theta_i^\top \right) }_{A} +D^2 \underbrace{\left( X_E\eta\eta^\top X_E^\top \right)}_{C}}.\label{matrix:theta_hat}
	\end{align}
	Define a set of new variables $z_i=\frac{1}{k} \hat \theta_i \hat \theta_i^\top - \frac{1}{k}\left(\theta_i \theta_i^\top + D \right)$. From \eqref{matrix:theta_hat}, we have 
	$
	z_i = \frac{1}{k} (DA+D^2C-D).
	$
	Then, the expectation of $z_i^2$ satisfies
	\begin{align}
		\BE z_i^2 = \frac{1}{k^2} \BE \big[&D^2A^2+D^4C^2 +D^2 +D^3(AC+CA) \nonumber\\
		&-2D^2 A-2D^3C \big], \label{Ez:sqrt}
	\end{align}
	where the fact that $D$ commutes with any matrix has been used.
	Since $D$ is deterministic, to compute $\BE z_i^2$, it suffices to calculate $\BE A^2, \BE C^2, \BE (AC+CA), \BE A,$ and $\BE C$.
	
	For $\BE A^2$, it holds that 
	\begin{align*}
		\BE A^2 = &\BE (\theta_i \eta^\top X_E^\top)^2 + \BE (X_E\eta\theta_i^\top)^2 + \BE (\theta_i \eta^\top X_E^\top X_E\eta\theta_i^\top ) 
		\\ &+ \BE (X_E \eta\theta_i^\top  \theta_i \eta^\top X_E^\top )\\
		= &2 \frac{\lambda_0^2 N_1}{d} \theta_i \theta_i^\top  + N_1\lambda_0^2  \theta_i \theta_i^\top+ \theta_i^\top \theta_i  \frac{\lambda_0^2  N_1}{d} I_d. 
	\end{align*}
	For $\BE C^2$, we have 
	\begin{align*}
		&\BE C^2  =   \BE \left[ X_E\eta\eta^\top X_E^\top X_E\eta\eta^\top X_E^\top\right] \\
		& {= \BE \left[ \eta^\top X_E^\top X_E\eta \cdot X_E\eta \eta^\top X_E^\top\right] = \BE \left[ \sum_{t=1}^{N_1} \eta_t^4 x_t^\top x_t x_t x_t^\top\right]}\\
		&{= \psi_4 \lambda_0^2  \left(\sum_{t=1}^{N_1}  x_t x_t^\top \right) = \psi_4 \lambda_0^2 X_E X_E^\top =\frac{\psi_4 \lambda_0^4 N_1}{d} I_d,}
	\end{align*}
	where $\psi_4=\BE \eta_t^4$ ($\psi_4$ always exists since $\eta_t$ is a sub-Gaussian random variable).	
	For $\BE (AC+CA)$, it holds that 
	\begin{align*}
		\BE& (AC+CA) = \BE \big[ \big( \theta_i \eta^\top X_E^\top +  X_E\eta\theta_i^\top \big) X_E\eta\eta^\top X_E^\top
		\\		+&  X_E\eta\eta^\top X_E^\top  \left( \theta_i \eta^\top X_E^\top +  X_E\eta\theta_i^\top \right)\big]\\
		= &\BE \left[\eta^\top X_E^\top  X_E\eta \cdot \theta_i \eta^\top X_E^\top  \right] + \BE \left[ \theta_i^\top X_E \eta \cdot X_E\eta \eta^\top X_E^\top\right]\\
		+& \BE \left[ \eta^\top X_E^\top \theta_i \cdot X_E \eta \eta^\top X_E^\top\right] +\BE \left[ \eta^\top X_E^\top  X_E\eta \cdot  X_E\eta \theta_i^\top\right]\\
		=& {\frac{\lambda_0^3 \psi_3 N_1}{ d} \theta_i \mathbf 1_d^\top+  \frac{\lambda_0^3 \psi_3 N_1}{ d} \diag{\theta_i} +  \frac{\lambda_0^3 \psi_3 N_1}{ d} \diag{\theta_i} 
			+\frac{\lambda_0^3 \psi_3 N_1}{ d}\mathbf 1_d  \theta_i ^\top }\\
		=&\frac{\lambda_0^3 \psi_3 N_1}{ d} (\theta_i \mathbf 1_d^\top+\mathbf 1_d  \theta_i ^\top) +\frac{2\lambda_0^3 \psi_3 N_1}{ d} \diag{\theta_i}.
	\end{align*}
	Notice that $\BE A =0$. For $\BE C$, it holds that
	$
	\BE C = \BE X\eta\eta^\top X^\top =  \frac{\lambda_0^2 N_1}{d} I_d.
	$
	
	Overall, substituting all the above terms into Eq.~\eqref{Ez:sqrt} we have 
	\begin{align*}
		\BE &z_i^2 = \frac{1}{k^2} \frac{d^2}{\lambda_0^4 N_1^2} \Big[ 2 \frac{\lambda_0^2  N_1}{d} \theta_i \theta_i^\top  + N_1\lambda_0^2  \theta_i \theta_i^\top \\
		&+ \theta_i^\top \theta_i  \frac{\lambda_0^2  N_1}{d} I_d + D^2 \frac{\psi_4 \lambda_0^4 N_1}{d} I_d +I_d+D\Big(\frac{\lambda_0^3 \psi_3 N}{ d}  \\
		&\cdot (\theta_i \mathbf 1_d^\top+\mathbf 1_d  \theta_i ^\top) +\frac{2\lambda_0^3 \psi_3 N}{ d} \diag{\theta_i} \Big)-2D\frac{\lambda_0^2  N}{d} I_d \Big]\\
		&\le \frac{d^2 }{k^2 \lambda_0^2 N_1} \big(\frac{2}{d}+1\big) \theta_i \theta_i^\top +O\big(\frac{d^2}{k^2 \lambda_0^4 N^2_1}I_d \big).
	\end{align*}
	
	Let $\sigma^2 = \|\sum_{i=1}^{k}\BE z_i^2 \|_F$, and it satisfies 
	\begin{align*}
		{\sigma^2  \lesssim  \left\|\frac{d^2 }{k \lambda_0^2 N_1} \frac{1}{k}\sum_{i=1}^k  \theta_i \theta_i^\top \right\|_F \le O \left(\frac{d^2 }{k \lambda_0^2 N_1}\Tr(W_k)\right)}.
	\end{align*}
	Since for any $\theta \in \mathcal T$, it holds that $\theta_{\min}\le \|\theta\| \le \theta_{\max}$ with  $\theta_{\min}=\Theta(1)$ and $\theta_{\max}=\Theta(1)$, it follows that $\Tr(W_k)=\Theta(1)$. Therefore, it holds that
	$
	{\sigma^2  \le O \left(\frac{d^2 }{k \lambda_0^2 N_1}\right)}.
	$
	
	Applying Lemma~\ref{lemma:sym-matr} with $t= 2 c_1 \log (2d kN_1)+c_2\sqrt{4 \sigma^2\log(2dkN_1)}$ for sufficiently large $c_1,c_2>0$, we have 
	\begin{align*}
		{\left\|\sum_{i=1}^{k}z_i \right\|_F \lesssim  \frac{d}{\lambda_0} \sqrt{\frac{1}{kN_1} } (\sqrt{\log \left({kdN_1} \right)} + \log \left({kdN_1} \right))}.
	\end{align*}
	with probability at least $1- \frac{1}{kN_1}$. 
	
	Notice that $\sum_{i=1}^{k}z_i=\hat W -(W+ D)$. Let $W'=W +D$, we have
	\begin{align*}
		{\left\| \hat W-W' \right\|_F:=\|\Delta\|_F \lesssim \frac{d }{\lambda_0} \sqrt{\frac{1}{kN_1} } (\sqrt{\log \left({kdN_1} \right)} + \log \left({kdN_1} \right) )}.
	\end{align*}
	
	\textbf{Step 2.} From the Davis-Kahan $\sin \bm{\theta}$ Theorem, we have
	\begin{align}\label{distance}
		{\|\hat B^\top B_\perp\|_F  \le \frac{\|\hat B_\perp^\top (W -W') B\|}{\omega} \le \frac{\|\hat W -W'\|_F}{\omega}},
	\end{align}
	where $\omega = \inf_{1\le i\le r,r<j\le d }|\lambda_i(W' )-\lambda_j(\hat W)|$. From the Weyl's Theorem, 
	$
	|\lambda_i( W' )-\lambda_i(\hat W)|\le \|\hat W -W'\|_F=\left\|\sum_{j=1}^{k}z_j \right\|_F
	$
	for any $i=1,\dots,d$. Since $\lambda_i(W')=0$ for all $i\ge r+1$, it holds that $|\lambda_i(\hat W)|\le \| \Delta \|_F$ for all $i\ge r+1$. Recall that $\sigma_r$ is the $r$-th largest eigenvalue of $W$, therefore $\omega \ge \sigma_k-\| \Delta \|_F$. From the Assumption 2, we know $\sigma_r \ge \nu/r$, therefore we obtain
	\begin{align*}
		{\|\hat B^\top B_\perp\|_F\lesssim \frac{\| \Delta \|_F}{\sigma_r-\| \Delta \|_F} 
			\lesssim
			\frac{d r}{\lambda_0 \nu}\sqrt{\frac{1}{ nL N_1}}\Big( \sqrt{\log dnLN_1}+  \log dnLN_1 \Big)},
	\end{align*} 
	where $k=nL$ have been used. The proof is complete.
\end{proof}

\subsubsection{Analysis of Lemmas~\ref{lemma:RepCD}~and~\ref{lemma:RepCD:estimate}}
Let us present some instrumental results first.

\begin{lemma}[Random projection, Chap.~5, \cite{VR-book:2008}]\label{Rand:Projec}
	Let $P$ be a projection from $\R^n$ onto a random $m$-dimensional subspace uniformly distributed in the Grassmann manifold $G_{n,m}$. Let $x\in \R^n$ be a fixed point and $\beta >0$. Then, with probability at least $1-\exp(-c \beta^2 m)$, we have
	\begin{align*}
		{(1-\beta) \sqrt{\frac{m}{n}} \|x\|_2 \le \|Px\|_2 \le (1+\beta) \sqrt{\frac{m}{n}}\|x\|_2}.
	\end{align*}
\end{lemma}

\begin{lemma}\label{lemma:randn-prj-bandit}
	Denote $z= B_\perp ^\top \theta$. Let $P$ be a projection matrix from $\R^{d-r}$ onto a random $m$-dimensional subspace uniformly distributed in the Grassmann manifold $G_{(d-r), m}$. Then, it holds with probability at least $1-\exp(-c\beta^2 m)$ that  
	\begin{align*}
		{\frac{1}{2}  \sqrt{\frac{m}{d-r}} \|z\|_2 \le \|Pz\|_2 \le \frac{3}{2}  \sqrt{\frac{m}{d-r}}\|z\|_2}.
	\end{align*}
	
\end{lemma}
The proof of Lemma~\ref{lemma:randn-prj-bandit} directly follows from Lemma~\ref{Rand:Projec} by letting $\beta =1/2$.

\begin{lemma}[Concentration of the norm, Chap.~3, \cite{VR-book:2008}]\label{concentration}
	Suppose that $X=[X_1,X_2,\dots, X_n]^\top$ is a random vector, where $X_1,\dots, X_n$ are independent $\delta$-sub-Gaussian random variable. Then, for any $\xi >0$ it holds that
	\begin{align}
		{\Pr \left[ \frac{1}{\sqrt{n} }\left| \|X\|_2 - \delta \right|\ge \xi \right] \le 2 \exp\left( -\frac{c n \xi^2 }{K^2} \right)},  \label{RepCD_detec}
	\end{align}
	where $c$ is an absolute constant and $K = \max_i \|X_i\|_{\psi_2}$ is assumed $K <1$.
\end{lemma}

We are now ready to prove Lemmas~\ref{lemma:RepCD} and~\ref{lemma:RepCD:estimate}. 

\begin{pfof}{Lemma~\ref{lemma:RepCD}}
	We construct the proof by showing that $Y_{\rm det}$ goes beyond $\mathcal C_{\rm det}$ with high probability when the task $\theta$ is played by  RepCD.
	
	%From Lemma~\ref{concentration}, the probability of $Y_{\rm det} \notin \mathcal C_{n_{\rm det}}$ is less than $2 \exp( -\frac{c n_{\rm det} \xi_{\rm det} }{K^2} )$ if $\theta =B\alpha$ for some alpha since $Y_{\rm det}=\lambda_0 P^\top B^\top _\perp\theta+{\eta} = {\eta}$. 
	%	

	Denote  $\rho=\|B^\top _\perp\theta\|$, and from  Lemma~\ref{lemma:randn-prj-bandit} we have $\|Q^\top  B^\top _\perp\theta\|=\|Q Q^\top  B^\top _\perp\theta\|\ge \frac{1}{2} \rho\sqrt{{\bar n}/{(d- r)}}$ with probability at least $1-\exp(-{c}\bar n/4)$ since $Q Q^\top  B^\top _\perp\theta$ can be taken as projecting $ B^\top _\perp\theta$ onto the random subspace spanned by $Q$. The reward vector $Y_{\rm det}$ satisfies $Y_{\rm det}=\lambda_0 G ^\top B^\top _\perp\theta+\eta$ with $G:=[\underbrace{I_{d-r},\dots, I_{d-r}}_k,Q]$. It follows that
	\begin{align*}
		&\Pr \left[ Y_{\rm det} \in \mathcal C_{\rm det}  \right]\\
		&\le {\Pr \Big[ \big| \frac{1}{\sqrt{n_{\rm det}} } \|\lambda_0 G^\top  B^\top _\perp\theta+\eta\| - 1  \big|\le  \xi_{\rm det}  \Big] }\\
		&\le {\Pr\Big[ \big| \frac{1}{\sqrt{n_{\rm det}} } \|\eta\| -1 \big| \ge  \frac{\lambda_0\rho}{2 \sqrt{d- r}}  \sqrt{1+ \frac{3k(d-r)}{n_{\rm det}}} - \xi_{\rm det} \Big]}\\
		&\le {\Pr\Big[ \big| \frac{1}{\sqrt{n_{\rm det}} } \|\eta\| -1 \big| \ge  \frac{\lambda_0\rho}{2 \sqrt{d- r}}  - \xi_{\rm det} \Big]}.
	\end{align*}
	Observe that $\|B^\top_\perp \theta\|=\|B^\top_\perp \bar B \alpha\|\ge \sigma_{\min} (B^\top_{\perp} \bar B)\|\alpha\|$. Since $\sigma_{\min} (B^\top_{\perp} \bar B)=\sin \bm{\theta}_r(B,\bar B)$ and $\theta\ge \theta_{\min}$, we have
	\begin{align*}
		\\|B^\top_\perp \theta\|	\ge \sin \bm{\theta}_r(B,\bar B) \theta_{\min} =  \kappa_1 \theta_{\min}.
	\end{align*}
	Therefore, we have $\rho \ge  \kappa_1  \theta_{\min}$. From lemma~\ref{concentration}, one can derive that
	\begin{align*}
		{\Pr \left[ Y_{\rm det} \in \mathcal C_{n_{\rm det}}  \right]\le 2\exp\Big(  -\frac{cn_{\rm det}( \frac{1}{2} \kappa_1 \lambda_0 \theta_{\min}   \sqrt{\frac{1}{d- r}} -  \xi_{\rm det} )^2  }{K^2}\Big)}.
	\end{align*}
	Let $n_{\rm net}=\frac{9(d-r)\log(2S^2 N)}{\kappa_1^2 \theta_{\min}^2 \lambda_0^2}$ and $\xi_{\rm det}=\sqrt{\frac{\log(2S^2 N)}{4 n_{\rm det}}}$. Then, it can be calculated that  $\Pr \left[ Y_{\rm det} \in \mathcal C_{\rm det} \right] \le O( \frac{1}{S^2 N})$, which means that the outlier $\theta$ to $B$ can be detected with probability at least $1-O(\frac{1}{S^2 N})$. The proof is complete.
\end{pfof}

For Lemma~\ref{lemma:RepCD}, the only difference is that only an estimate $\hat B$ satisfying $\|\hat B^\top_\perp B\|\le \varepsilon$ is known. It can be derived that $\|\hat B^\top_\perp \theta\| =\|B^\top_\perp \bar B \alpha\| \ge (\kappa_1-\varepsilon) \theta_{\min} $. Then, following similar steps as those for Lemma~\ref{lemma:RepCD}, one can prove lemma~\ref{lemma:RepCD:estimate}.

\subsubsection{Proof of Lemma~\ref{lemma:fasle-detection}}\label{pf:lm5}

\begin{proof}
	It follows from the proof of Theorem~\ref{dist:B} that, after the initial $bL$ tasks, the estimated representation $\hat B$ satisfies
	\begin{align*}
		\|\hat B^\top B_\perp\|_F \lesssim \frac{d r}{\lambda_0 \nu\sqrt{ b \ell N_1} }\big( \sqrt{\log (d b \ell N_1)}+  \log (d b\ell N_1 )\big).
	\end{align*}
	For the simplicity of notation, let $q=\sqrt{\log (d b \ell N_1)}+  \log (d b\ell N_1 )$. 
	For $\theta$, there is $\alpha \in \R^r$ such that $\theta=B\alpha$, which implies that 
	$
	Y_{\rm det} = \lambda_0 G^\top \hat B^\top _\perp\theta+\eta = \lambda_0 G^\top \hat B^\top _\perp B\alpha+\eta,
	$
	where $G=[I_{d-r},\dots,I_{d-r},Q]$. Since $\| Q^\top \hat B_\perp^\top \theta \|=\|QQ^\top \hat B_\perp^\top \theta\|$, it follows from Lemma~\ref{lemma:randn-prj-bandit} that $\|Q^\top  \hat B^\top _\perp\theta\|\le \frac{3}{2} \|\hat B_\perp^\top \theta\|\sqrt{{\bar n}/{(d- r)}}$ with probability at least $1-\exp(-{c}\bar n/4)$. Denote $\varepsilon=\|\hat B_\perp^\top B\|_F$, and it can be observed that $\|\hat B_\perp^\top \theta\|\le \|\hat B_\perp^\top B\|_F\cdot \|\alpha\|\le \varepsilon \theta_{\max}$. Subsequently, since $n_{\rm det}=k(d-r)+\bar n$, it holds that
	\begin{align*}
		\frac{1}{\sqrt{n_{\rm det}}}\|Y_{\rm det}\|_2 &\le \frac{1}{\sqrt{n_{\rm det}}} \big(\frac{\lambda_0 \varepsilon \theta_{\max}}{\sqrt{d-r}} \sqrt{k(d-r)+\frac{9}{4}\bar n}+\|\eta\|\big)\\
		&\le{\frac{3\lambda_0 \varepsilon \theta_{\max}}{2\sqrt{d-r}} + \frac{1}{\sqrt{n_{\rm det}}} \|\eta\|\le  \frac{1}{\sqrt{n_{\rm det}}} \|\eta\|+\underbrace{ \frac{3d r \theta_{\max}q}{2 \nu\sqrt{ b L N_1(d-r)}  }}_{s}}.
	\end{align*}
	Then, from Lemma~\ref{concentration}, we have 
	\begin{align*}
		{\Pr [ Y_{\rm det} \notin \mathcal C_{\rm det}|\theta=B\alpha] \le 2 \exp\left( -\frac{c n_{\rm det} (\xi_{\rm det}-s)^2 }{K^2} \right)}.
	\end{align*}
	Substituting Eqs.~\eqref{value:b} and \eqref{det_selection:m} into the right-hand side yields $\Pr [\theta \notin \Span(\hat B)|\theta=B\alpha]\le \frac{1}{S^2 N}$ given $N_1=dr\sqrt{N/\ell}$. 
\end{proof}

\section*{ACKNOWLEDGMENT}
This work was supported in part by under Award ARO-78259-NS-MUR and Award AFOSR-FA9550-20-1-0140.

%\section*{References}

\bibliographystyle{unsrt}
\bibliography{alias,FP,Main,New}

%%%%%%%%%%%%%%%%%%%%%%%%%%%%%%%%%%%%%%%%%%%%%%%%%%%%%%%%%%%%

\end{document}